\documentclass[11pt]{article}
\pdfoutput=1
\usepackage{enumerate}
\usepackage{pdfsync}
\usepackage[OT1]{fontenc}

\usepackage[usenames]{color}
\usepackage{smile}
\usepackage[colorlinks,
            linkcolor=red,
            anchorcolor=blue,
            citecolor=blue
            ]{hyperref}
\usepackage{fullpage}
\usepackage{hyperref}
\usepackage[protrusion=true,expansion=true]{microtype}
\usepackage{pbox}
\usepackage{setspace}
\usepackage{tabularx}
\usepackage{float}
\usepackage{wrapfig,lipsum}
\usepackage{enumitem}
\usepackage{microtype}
\usepackage{graphicx}
\usepackage{subfigure}
\usepackage{enumerate}
\usepackage{enumitem}
\usepackage{pgfplots}
\usetikzlibrary{arrows,shapes,snakes,automata,backgrounds,petri}
\usepackage{booktabs} 

\allowdisplaybreaks
\newcommand{\la}{\langle}
\newcommand{\ra}{\rangle}
\newcommand{\qvalue}{Q}
\newcommand{\vvalue}{V}

\newcommand{\reward}{r}

\usepackage{enumitem}

\makeatletter
\newcommand*{\rom}[1]{\expandafter\@slowromancap\romannumeral #1@}
\makeatother
\title{\huge Logarithmic Regret for Reinforcement Learning with Linear Function Approximation}

\author
{
	Jiafan He\thanks{Department of Computer Science, University of California, Los Angeles, CA 90095, USA; e-mail: {\tt jiafanhe19@ucla.edu}} 
	~~~and~~~
	Dongruo Zhou\thanks{Department of Computer Science, University of California, Los Angeles, CA 90095, USA; e-mail: {\tt drzhou@cs.ucla.edu}} 
	~~~and~~~
	Quanquan Gu\thanks{Department of Computer Science, University of California, Los Angeles, CA 90095, USA; e-mail: {\tt qgu@cs.ucla.edu}}
}
\date{}
\usepackage{smile}

\begin{document}
\maketitle




\begin{abstract}%
Reinforcement learning (RL) with linear function approximation has received increasing attention recently. 
However, existing work has focused on obtaining $\sqrt{T}$-type regret bound, where $T$ is the number
of interactions with the MDP. 
In this paper, we show that logarithmic regret is attainable under two recently proposed linear MDP assumptions provided that there exists a positive sub-optimality gap for the optimal action-value function. More specifically, under the linear MDP assumption \citep{jin2019provably}, the LSVI-UCB algorithm can achieve $\tilde{O}(d^{3}H^5/\text{gap}_{\text{min}}\cdot \log(T))$ regret; and under the linear mixture 
MDP assumption \citep{ayoub2020model}, the UCRL-VTR algorithm can achieve $\tilde{O}(d^{2}H^5/\text{gap}_{\text{min}}\cdot \log^3(T))$ regret, where $d$ is the dimension of feature mapping, $H$ is the length of episode, $\text{gap}_{\text{min}}$ is the minimal sub-optimality gap, and $\tilde O$ hides all logarithmic terms except $\log(T)$. To the best of our knowledge, these are the first logarithmic regret bounds for RL with linear function approximation. We also establish gap-dependent lower bounds for the two linear MDP models.
\end{abstract}


\section{Introduction}

Designing efficient algorithms that learn and plan in sequential decision-making tasks with large state and action spaces has become a central task of modern reinforcement learning (RL) in recent years. RL often assumes the environment as a Markov Decision Process (MDP), described by a tuple of state space, action space, reward function, and transition probability function. Due to a large number of possible states and actions, traditional tabular reinforcement learning methods such as Q-learning \citep{watkins1989learning}, which directly access each state-action pair, are computationally intractable. A common approach to cope with high-dimensional state and action spaces is to utilize function approximation such as linear functions or neural networks to map states and actions to a low-dimensional space.

Recently, a large body of literature has been devoted to provide regret bounds for online RL with linear function approximation. These works can be divided into two main categories. The first category is \emph{model-free} algorithms, which directly parameterize the action-value function as a linear function of some given feature mapping. For instance, \citet{jin2019provably} studied the episodic MDPs with \emph{linear MDP} assumption, which assumes that both transition probability function and reward function can be represented as a linear function of a given feature mapping. Under this assumption, \citet{jin2019provably} showed that the action-value function is a linear function of the feature mapping and proposed a model-free LSVI-UCB algorithm to obtain an $\tilde O(\sqrt{d^3 H^3T})$ regret, where $d$ is the dimension of the feature mapping, $H$ is the length of the episode, and $T$ is the number
of interactions with the MDP. The second category is \emph{model-based} algorithms, which parameterize the underlying transition probability function as a linear function of a given feature mapping. For example, \citet{ayoub2020model} 
studied a class of MDPs where the underlying
transition probability kernel is a linear mixture model. \citet{ayoub2020model} proposed a model-based UCRL-VTR algorithm with an $\tilde O(d\sqrt{H^3T})$ regret. \citet{zhou2020provably} studied the linear kernel MDP\footnote{Linear kernel MDPs are essentially the same as linear mixture MDPs.} 
in the infinite horizon discounted setting and proposed a algorithm with $\sqrt{T}$-type regret. Although these $\sqrt{T}$-type regrets are standard and easy to interpret, they do not consider any additional \emph{problem-dependent structure} of the underlying MDPs. This motivates us to seek a tighter and instance-dependent regret analysis for RL.

There is a large body of literature on bandits, which study the instance-dependent regret bounds (See  \citet{bubeck2012regret,slivkins2019introduction,lattimore2018bandit} and references therein). Note that bandits can be seen as a special instance of RL problems. Sub-optimality gap has been playing a central role in many gap-dependent bounds for bandits, which is defined as gap between the optimal action and the rest ones. 
 For general RL, previous works \citep{simchowitz2019non, yang2020q} have considered the tabular MDP with sub-optimality gap and proved gap-dependent regret bounds. However, as far as we know, there does not exist such gap-dependent regret results for RL with linear function approximation. Therefore, a natural question arises: 

\begin{center}
    \emph{Can we derive instance/gap-dependent regret bounds for RL with linear function approximation?}
\end{center}

    
We answer the above question affirmatively in this paper. In detail, following \citet{simchowitz2019non, yang2020q}, we consider an instance-dependent quantity called \emph{$\text{gap}_{\text{min}}$}, which is the minimal sub-optimality gap for the optimal action-value function. 
Under the assumption that $\text{gap}_{\text{min}}$ is strictly positive, we show that LSVI-UCB proposed in \citet{jin2019provably} achieves a $\tilde{O}(d^{3}H^5/\text{gap}_{\min}\cdot \log(T))$ regret, and UCRL-VTR proposed by \citet{ayoub2020model} achieves a regret of order $\tilde{O}(d^2H^5/\text{gap}_{\min}\cdot \log^3(T))$. Furthermore, we prove an $\Omega(dH/\text{gap}_{\min})$ lower bound on the regret for both linear MDPs and linear mixture MDPs. To the best of our knowledge, this is the first instance-dependent $\log T$-type regret achieved by RL with linear function approximation. 
Our results suggest that the dependence on $T$ in regrets can be drastically decreased from $\sqrt{T}$ to $\log T$ when considering the problem structure for both model-free and model-based RL algorithms with linear function approximation.

\noindent\textbf{Notation} 
We use lower case letters to denote scalars, and use lower and upper case bold face letters to denote vectors and matrices respectively. For any positive integer $n$, we denote by $[n]$ the set $\{1,\dots,n\}$. For a vector $\xb\in \RR^d$ , we denote by $\|\xb\|_1$ the Manhattan norm and denote by $\|\xb\|_2$ the Euclidean norm. For a vector $\xb\in \RR^d$ and matrix $\bSigma\in \RR^{d\times d}$, we define $\|\xb\|_{\bSigma}=\sqrt{\xb^\top\bSigma\xb}$. For two sequences $\{a_n\}$ and $\{b_n\}$, we write $a_n=O(b_n)$ if there exists an absolute constant $C$ such that $a_n\leq Cb_n$. We use $\tilde O(\cdot)$ to further hide the logarithmic factors. For logarithmic regret, we use $\tilde O(\cdot)$ to hide all logarithmic terms except $\log T$.
\section{Related Work}


\noindent\textbf{Episodic tabular MDP.} There are a series of work focusing on regret or sample complexity of online RL on tabular episodic MDPs. They can be recognized as model-free methods or model-based methods, which depend on whether they explicitly estimate the model (transition probability function) or not. \citet{dann2015sample} proposed a UCFH algorithm that adopts a variant of extended value iteration and obtains a polynomial sample complexity. \citet{azar2017minimax} proposed a UCB-VI algorithm which adapts the Bernstein-style exploration bonus with a $\tilde{O}(\sqrt{HSAT})$ regret which matches the lower bound proposed in \cite{jaksch2010near, osband2016lower} up to logarithmic factors. \citet{zanette2019tighter} proposed an EULER algorithm which utilizes the problem-dependent bound and achieves a $\tilde{O}(\sqrt{HSAT})$ regret. For model-free algorithms, \citet{strehl2006pac} proposed a delayed Q-learning algorithm with a sublinear regret. Later, \citet{jin2018q} proposed a Q-learning with UCB algorithm which achieves $\tilde{O}(\sqrt{H^3SAT})$ regret. Recently, \cite{zhang2020almost} proposed a UCB-advantage algorithm with an improved regret $\tilde{O}(\sqrt{H^2SAT})$, which matches the information-theoretic lower bound \citep{jaksch2010near, osband2016lower} up to logarithmic factors.

\noindent\textbf{Logarithmic regret bound for RL.}
\noindent A line of works focus on proving $\log T$-style regret bounds for RL algorithms based on problem-dependent quantities. It has been shown that such a $\log T$ dependence is unavoidable according to the lower bound results shown in \citet{ok2018exploration}. For the upper bounds,  \citet{auer2007logarithmic} showed that the UCRL algorithm achieves logarithmic regret in the average reward setting, while the regret bound depends on both the hitting time and the policy sub-optimal gap.
\citet{tewari2008optimistic} proposed an OLP algorithm for average-reward MDP and showed that OLP achieves logarithmic
 regret $O\big(C(P)\log T\big)$ where $C(P)$ is an explicit MDP-dependent constant. Both results in \citet{auer2007logarithmic} and \citet{tewari2008optimistic} are asymptotic, which required the number of steps $T$ is large enough. For non-asymptotic bounds, \citet{jaksch2010near} proposed a UCRL2 algorithm for average-reward MDP with regret $O\big(D^2S^2A\log(T)/\Delta\big)$, where $D$ is the diameter of the MDP and $\Delta$ is the policy sub-optimal gap. 
 For episodic MDPs, \citet{simchowitz2019non} proposed a model-based StrongEuler algorithm with a logarithmic regret, 
 and proved a regret lower bound for tabular MDPs that depends on the minimal sub-optimality gap. Recently, \citet{yang2020q} showed that the model-free algorithm optimistic Q-learning achieves $O\big(SAH^6\log(SAT)/\text{gap}_{\min}\big)$ regret. However, all the above results are limited to tabular MDPs.

\noindent\textbf{Linear function approximation.}
Recently, there has emerged a large body of literature on learning MDPs with linear function approximation. 
These results can be categorized based on their assumptions on the MDPs. The first category of works consider linear MDPs \citep{yang2019sample,jin2019provably}. 
\citet{jin2019provably} proposed LSVI-UCB algorithm with $\tilde O(\sqrt{d^3H^3T})$ regret. \citet{wang2019optimism} proposed USVI-UCB algorithm in a weaker assumption called “optimistic closure” and achieved $\tilde{O}(H\sqrt{d^3T})$ regret. \citet{zanette2020learning} proposed a weaker assumption which is called low inherent Bellman error, and improved the regret to $\tilde O(dH\sqrt{T})$ by considering a global planning oracle. 
\citet{jiang2017contextual} studied a larger class of MDPs with low Bellman rank and proposed an OLIVE algorithm with polynomial sample complexity. 
The second line of works consider linear mixture MDPs \citep{jia2020model,ayoub2020model}. \citet{jia2020model} and \citet{ayoub2020model} proposed UCLR-VTR algorithm for episodic MDPs which achieves $\tilde O(d\sqrt{H^3T})$ regret. \citet{cai2019provably} proposed policy optimization algorithm OPPO which achieves $\tilde{O}(\sqrt{d^2H^3T})$ regret.
\citet{zhou2020provably} focused on infinite-horizon discounted setting and proposed a UCLK algorithm, which achieves $\tilde O(d\sqrt{T}/(1-\gamma)^2)$ regret.

\vskip -0.3in
\section{Preliminaries}
In this paper, we consider episodic  Markov Decision Processes (MDP) which can be denoted by a tuple $\cM(\cS, \cA, H, \{\reward_h\}_{h=1}^H, \{\PP_h\}_{h=1}^H)$. Here, $\cS$ is the state space, $\cA$ is the finite action space,  $H$ is the length of each episode, $\reward_h: \cS \times \cA \rightarrow [0,1]$ is the reward function at step $h$ and $\PP_h(s'|s,a) $ is the transition probability function at step $h$ which denotes the probability for state $s$ to transfer to state $s'$ with action $a$ at step $h$. 

A  policy $\pi: \cS \times [H] \rightarrow \cA$ is a function which maps a state $s$ and the 
step number 
$h$ to an action $a$. 
For any policy $\pi$ and step $h\in [H]$, we denote the action-value function $\qvalue_h^{\pi}(s,a)$ and value function $\vvalue_h^{\pi}(s)$ as follows
\begin{align}
\qvalue^{\pi}_h(s,a) &=\reward_h(s,a) + \EE\bigg[\sum_{h'=h+1}^\infty \reward_{h'}\big(s_{h'}, \pi(s_{h'},h')\big)\bigg],\notag\\
    \vvalue_h^{\pi}(s) &= \qvalue_h^{\pi}(s, \pi(s,h)),\notag
\end{align}
where $s_h=s,a_h=a$ and $s_{h'+1}\sim \PP_h(\cdot|s_{h'},a_{h'})$.
We define the optimal value function $V_h^*$ and the optimal action-value function $\qvalue_h^*$ as $V_h^*(s) = \sup_{\pi}\vvalue_h^{\pi}(s)$ and $\qvalue^*(s,a) = \sup_{\pi}\qvalue^{\pi}(s,a)$. By definition, the value function $\vvalue_h^{\pi}(s)$ and action-value function $\qvalue_h^{\pi}(s,a)$ are bounded in $[0,H]$. 
For simplicity, for any function $\vvalue: \cS \rightarrow \RR$, we denote $[\PP_h \vvalue](s,a)=\EE_{s' \sim \PP_h(\cdot|s,a)}\vvalue(s')$. Therefore, for each $h\in[H]$ and policy $\pi$, we have the following Bellman equation, as well as the Bellman optimality equation:
\begin{align}
    \qvalue_h^{\pi}(s,a) &= \reward_h(s,a) + [\PP_h\vvalue_{h+1}^{\pi}](s,a),\notag\\
    \qvalue_h^{*}(s,a) &= \reward_h(s,a) + [\PP_h\vvalue_{h+1}^{*}](s,a),\label{eq:bellman}
\end{align}
where $\vvalue^{\pi}_{H+1}=\vvalue^{*}_{H+1}=0$. 
In the \emph{online learning setting}, for eack episode $k\ge 1$, at the beginning of the episode $k$, the agent determine a policy $\pi_k$ to be followed in this episode. At each step $h\in[H]$, the agent observe the state $s_h^k$, choose an action following the policy $\pi_k$ and observe the next state with $s_{h+1}^k \sim \PP_h(\cdot|s_h^k,a_h^k)$. Furthermore, we define the total regret in the first $K$ episodes as follows. 
\begin{definition}\label{def:Regret} 
For any algorithm, we define its regret on MDP $M(\cS, \cA, H, \reward, \PP)$ in the first $K$ episodes as the sum of the suboptimality for epsiode $k = 1,\ldots, K$, i.e.,
\begin{align}
    \text{Regret}(K) = \sum_{k=1}^K \vvalue_1^*(s_1^k) - \vvalue_1^{\pi_k}(s_1^k),\notag
\end{align}
where $\pi_k$ is the policy in the episodes $k$.
\end{definition}
In this paper, we focus on the minimal sub-optimality gap condition \citep{simchowitz2019non,du2019provably,du2020agnostic,yang2020q,mou2020sample} and linear function approximation \citep{jin2019provably,ayoub2020model,jia2020model,zhou2020provably}.
\begin{definition}[Minimal sub-optimality gap]
For each $s\in\cS, a\in\cA$ and step $h\in[H]$, the sub-optimality gap $\text{gap}_h(s,a)$ is defined as
\begin{align}
    \text{gap}_h(s,a)=\vvalue_h^*(s)-\qvalue^*_h(s,a),\notag
\end{align}
and the minimal sub-optimality gap is defined as
\begin{align}
    \text{gap}_{\min}=\min_{h,s,a}\big\{\text{gap}_h(s,a): \text{gap}_h(s,a)\ne 0\big\}.\label{definition-gap-min}
\end{align}
\end{definition}

\noindent In this paper, we assume the minimal sub-optimality gap is strictly positive.
\begin{assumption}\label{assumption:gap}
The minimal sub-optimality gap is strictly positive, i.e., $\text{gap}_{\min}>0$.
\end{assumption}

\section{Model-free RL}
In this section, we focus on model-free RL algorithms with linear function approximation.
We make the following linear MDP assumption \citep{jin2019provably, yang2019sample} where the probability transition kernels and the reward functions are assumed to be linear with respect to a given feature mapping $\bphi: \cS \times \cA \rightarrow \RR^d$. 
\begin{assumption}\label{assumption-linear}
MDP $\cM(\cS, \cA, H, \{\reward_h\}_{h=1}^H, \{\PP_h\}_{h=1}^H)$ is a linear MDP such that for any step $h\in[H]$, there exists an unknown vector $\bmu_h$, unknown measures $\btheta_h=\big(\btheta_h^{(1)},..,\btheta_h^{(d)}\big)$ and a known feature mapping $\bphi: \cS \times \cA \rightarrow \RR^d$, where for each $(s,a) \in \cS \times \cA$ and $s' \in \cS$, 
\begin{align}
    \PP_h(s'|s,a) = \big\la \bphi(s,a), \btheta_h(s')\big\ra,\reward_h(s,a)=\big\la \bphi(s,a),\bmu_h\big\ra.\notag
\end{align}
\end{assumption}
For simplicity, we assume that the unknown vector $\bmu_h$ and feature $\bphi(s,a)$ satisfy $\|\bphi(s,a)\|_2 \leq 1$, $\|\bmu_h\|_2 \leq \sqrt{d}$ and $\big\|\btheta_h(\cS)\big\|\leq \sqrt{d}.$
\begin{remark}\label{remark:linearq}
Under Assumption \ref{assumption-linear}, by the Bellman equation \eqref{eq:bellman}, it can be shown that for any policy $\pi$, the action-value function $\qvalue^\pi_h(s,a)$ is a linear function $\la \bphi(s,a), \btheta^\pi_h\ra$ with respect to the feature mapping $\bphi$, where $\btheta_h^\pi$ is a vector decided by the policy $\pi$. This suggests to estimate the unknown optimal action-value function $\qvalue^*_h$, we only need to estimate its corresponding parameter $\btheta^*_h$.
\end{remark}
\begin{remark}
Though the probability transition kernel and the reward function are linear with respect to $\bphi(s,a)$, the degree of freedom of measure $\btheta_h$ is $|\cS|\times d$. Therefore, when $\cS$ is large, it is computationally intractable to directly estimate the probability transition kernel $\PP_h$. 
\end{remark}
\subsection{Algorithm}

We analyze the LSVI-UCB algorithm proposed in \citet{jin2019provably}, which is showed in Algorithm \ref{algorithm1}.  At a high level, Algorithm \ref{algorithm1} treats the optimal action-value function $Q^*_h$ as a linear function of the feature $\bphi$ and an unknown parameter $\btheta_h^*$. The goal of Algorithm \ref{algorithm1} is to estimate $\btheta_h^*$.
Algorithm~\ref{algorithm1} directly estimates the action-value function, and that is why it is ``model-free".
Algorithm \ref{algorithm1} uses the least-square value iteration to estimate the $\btheta_h^*$ for each $h$ with additional exploration bonuses. In Line \ref{algorithm:line1}, Algorithm \ref{algorithm1} computes $\wb_h^k$, the estimate of $\btheta_h^*$, by solving a regularized least-square problem:
\begin{align}
\wb_h^k&\leftarrow \argmin_{\wb_h^k\in \RR^d}\lambda\|\wb_h^k\|_2^2+\sum_{i=1}^{k-1}\big(\bphi(s_h^i,a_h^i)^\top\wb_h^k-\reward_h(s_h^i,a_h^i)-\max_{a}\qvalue_{h+1}^{k}(s_{h+1}^{i},a)\big)^2.\notag
\end{align}
In Line \ref{algorithm:line3}, Algorithm \ref{algorithm1} computes the action-value function $\qvalue_h^k(s,a)$ by $\wb_h^k$ and adds a $\text{UCB}$ bonus to make sure the estimate of action-value function $\qvalue_h^k(s,a)$ is an upper bound of the optimal action-value function $\qvalue_h^*(s,a)$. 
In Line \ref{algorithm:line4}, 
a greedy policy with respect to estimated action-value function $\qvalue_h^k(s,a)$ is used to choose action and transit to the next state.
\begin{algorithm*}[t]
    \caption{Least Square Value-iteration with UCB (LSVI-UCB) \citep{jin2019provably}}\label{algorithm1}
    \begin{algorithmic}[1]
\FOR{ episodes $k=1,\ldots,K$}
\STATE Received the initial state $s_1^k$. 
    \FOR{ step $h=H,\ldots,1$}
        \STATE $\Lambda_h^k=\sum_{i=1}^{k-1}\bphi(s_h^i,a_h^i)\bphi(s_h^i,a_h^i)^{\top}+\lambda \cdot \Ib$ \label{algorithm:line2}
        \STATE 
        $\wb_h^k=(\Lambda_h^k)^{-1}\sum_{i=1}^{k-1}\bphi(s_h^i,a_h^i)\big[\reward_h(s_h^i,a_h^i)+\max_{a}\qvalue_{h+1}^{k}(s_{h+1}^{i},a)\big]$\label{algorithm:line1}
        \STATE 
        $\qvalue_h^k(s,a)=\min\big\{\beta\sqrt{\bphi(s,a)^{\top}(\Lambda_h^k)^{-1}\bphi(s,a)}+\wb_h^{\top}\bphi(s,a),H\big\}$\label{algorithm:line3}
    \ENDFOR
    \FOR{ step $h=1,\ldots,H$}
    \STATE Take action $a_h^k\leftarrow \argmax_{a} \qvalue_h^k(s_h^k,a)$ and receive next state $s_{h+1}^k$\label{algorithm:line4}
    \ENDFOR
\ENDFOR
    \end{algorithmic}
\end{algorithm*}

\subsection{Regret Analysis}\label{section:main}
In this subsection, we present our regret analysis for LSVI-UCB. For simplicity, we denote $T = KH$, which is the total number of steps.
\begin{theorem}\label{thm:1} Under Assumptions \ref{assumption:gap} and  \ref{assumption-linear}, there exists a constant $C$ such that, if we set $\lambda=1$, $\beta=78dH\sqrt{\log(2dT/\delta)}$ in Algorithm \ref{algorithm1}, then with probability at least $1-2(K+1)H\log(H/\text{gap}_{\min})\delta-\log T\delta$, the  regret for Algorithm \ref{algorithm1} in first $T$ steps is upper bounded by
\begin{align}
    \text{Regret}(K)\leq \frac{9Cd^3H^5 \log(2dT/\delta)}{\text{gap}_{\min}} \iota+\frac{16H^2\log \delta}{3},\notag
\end{align}
where $\iota$ is defined as follows:
\begin{align}
     \iota = \log \bigg(\frac{Cd^3H^4 \log(2dT/\delta)}{\text{gap}^2_{\min}}\bigg).\notag
\end{align}
\end{theorem}
\begin{remark} If we set the $\delta$ in Theorem \ref{thm:1} as $\delta=1/(2K(K+1)H^3)$ and define the high probability event $\Omega$ as: $\{\text{Theorem } \ref{thm:1} \text{ holds}\}$. Then, for the expected regret, we have
\begin{align} 
    \EE\big[\text{Regret}(K)\big]&\leq\EE\big[\text{Regret}(K)|\Omega\big]\Pr[\Omega]+T\Pr[\bar{\Omega}]\notag\\
    &\leq \frac{9Cd^3H^5 \log(2dT/\delta)}{\text{gap}_{\min}} \iota +\frac{16H^2\log \delta}{3} +T\Pr[\bar{\Omega}]
    \notag\\
&=\tilde O(d^3H^5/\text{gap}_{\min} \log T)\notag.
\end{align}
\end{remark}
The regret bound in Theorem \ref{thm:1} is independent of the size of the state space $\cS$, action space $\cA$, and is only logarithmic in the number of steps $T$, which suggests that Algorithm \ref{algorithm1} is sample efficient for MDPs with large state and action spaces. To our knowledge, this is the first theoretical result that achieves logarithmic regret for model-free RL with linear function approximation.
Besides, the UCB bonus parameter $\beta$ depends on $T$ logarithmically. 
When the number of steps $T$ is unknown at the beginning, we can use the ``doubling trick'' \citep{besson2018doubling} to learn $T$ adaptively, and the regret will only be increased by a constant factor.

The following theorem gives a lower bound of the regret for any algorithm learning linear MDPs.
\begin{theorem}\label{thm:3}
Suppose $\text{gap}_{\min}\leq 1/(3dH), H\ge 3,$ then for any algorithm,  there exist a linear MDP such that expected regret is lower bounded by
\begin{align}
    \EE\big[\text{Regret(K)}\big]\ge\Omega\bigg(\frac{Hd}{\text{gap}_{\min}}\bigg).\notag
\end{align}
\end{theorem}


\section{Model-based RL}
In this section we focus on model-based RL with linear function approximation. We make the following linear mixture MDP assumption \citep{jia2020model,ayoub2020model,zhou2020provably}, which assumes that the unknown transition probability function is an aggregation of several known basis models. 
\begin{assumption}\label{assumption-linear-mixture}
MDP $\cM(\cS, \cA, H, \{\reward\}_{h=1}^{H}, \{\PP_h\}_{h=1}^{H})$ is called a linear mixture MDP if there exists an unknown vector $\btheta^*_h \in \RR^d$ with $\|\btheta^*_h\|_2\leq C_{\btheta}$ and a known feature mapping $\bphi(s'|s,a): \cS \times \cA \times \cS \rightarrow \RR^d$, such that
\begin{itemize} 
    \item For any state-action-next-state triplet $(s,a,s') \in \cS \times \cA \times \cS$, we have $\PP_h(s'|s,a) = \la \bphi(s'|s,a), \btheta_h^*\ra$; Moreover, the reward function $r$ is \emph{deterministic and known}.
    \item For any bounded function $\vvalue: \cS \rightarrow [0,1]$ and any tuple $(s,a)\in \cS \times \cA$, we have $\|\bphi_{{\vvalue}}(s,a)\|_2 \leq 1$, where $\bphi_{{\vvalue}}(s,a) = \sum_{s'\in\cS}\bphi(s'|s,a)\vvalue(s') \in \RR^d$. 
\end{itemize}
\end{assumption}
\subsection{Algorithm} 
In this subsection, we analyze the model-based UCRL with the Value-Targeted Model Estimation (UCRL-VTR) algorithm proposed in \citet{jia2020model, ayoub2020model}, which is shown in Algorithm \ref{algorithm2}. It is worth noting that the original UCRL-VTR algorithm is designed for the time-homogeneous MDP, where the transition probability functions $\PP_h$ are identical across different step $h$. In this paper, we consider the time-inhomogeneous MDP and therefore propose the following time-inhomogeneous version of UCRL-VTR algorithm, which is slightly different from the original algorithm.
At a high level, unlike Algorithm \ref{algorithm1} which treats the action-value function as a linear function, Algorithm \ref{algorithm2} treats the transition probability function as a linear function of the feature mapping $\bphi(\cdot|\cdot,\cdot)$ and an unknown parameter $\btheta^*$. The goal of Algorithm \ref{algorithm2} is to estimate $\btheta^*$, which makes Algorithm \ref{algorithm2} a model-based algorithm since it directly estimates the underlying transition model. To estimate $\btheta^*$, Algorithm \ref{algorithm2} computes the estimate $\btheta_{k+1}$ by solving the following regularized least-square problem in Line \ref{algorithm2:line11}:
\begin{align}
\btheta_{k+1}&\leftarrow \argmin_{\btheta\in \RR^d}\lambda\|\btheta\|_2^2+\sum_{i=1}^{k}\big(\bphi_{\vvalue_{h+1}^i}(s_h^i,a_h^i)^\top\btheta-\vvalue_{h+1}^i(s_{h+1}^i)\big)^2,\notag
\end{align}
where for any value function $\vvalue: \cS \rightarrow \RR$, we denote $\bphi_{{\vvalue}}(s,a) = \sum_{s'\in \cS}\bphi(s'|s,a)\vvalue(s') \in \RR^d$. The close-form solution to $\btheta_{k+1}$ can be computed by considering the accumulated covariance matrix $\bSigma_1^{k+1}$ in Line \ref{algorithm2:line12} and \ref{algorithm2:line13}. 
To guarantee exploration, in Line \ref{algorithm2:line3}, Algorithm \ref{algorithm2} computes the action-value function $\qvalue_h^{k+1}(s,a)$ by $\btheta_{k+1}$ and adds a $\text{UCB}$ bonus to make sure the estimate of action-value function $\qvalue_h^{k+1}(s,a)$ is an upper bound of the optimal action-value function $\qvalue_h^*(s,a)$. Algorithm \ref{algorithm2} then follows the greedy policy induced by the estimated action-value function $\qvalue_h^k(s,a)$ in Line \ref{algorithm2:line3}.


\begin{algorithm}[t]
    \caption{ UCRL with Value-Targeted Model Estimation (UCRL-VTR) \citep{jia2020model, ayoub2020model}}
    \begin{algorithmic}[1]\label{algorithm2}
    \STATE Set $\bSigma_h^1=\lambda\Ib$, $\bbb_h^1=\zero$
\FOR{ episodes $k=1,\ldots,K$}
    \STATE Compute $ \btheta_{k,h}\leftarrow (\bSigma_{h}^{k})^{-1}\bbb_{h}^{k}$\label{algorithm2:line11}
    \FOR{ step $h=H,\ldots,1$}
        \STATE 
        $\qvalue_h^{k}(s,a)=\reward(s,a)+\bphi_{\vvalue_{h+1}^{k}}(s,a)^{\top}\btheta_{k,h}+\beta_{k}\sqrt{\big(\bphi_{\vvalue_{h+1}^{k}}(s,a)\big)^{\top}(\bSigma_{h}^{k})^{-1}\bphi_{\vvalue_{h+1}^k}(s,a)}$\label{algorithm2:line3}
    \ENDFOR
\STATE Received the initial state $s_1^k$ 
    \FOR{ step $h=1,\ldots,H$}
    \STATE Take action $a_h^k\leftarrow \argmax_{a} \qvalue_h^k(s_h^k,a)$ and receive next state $s_{h+1}^k$\label{algorithm2:line4}
    \STATE Update value matrix $\bSigma$ and vector $\bbb$: 
    \STATE $\bSigma_{h}^{k+1}\leftarrow \bSigma_{h}^{k}+\big(\bphi_{\vvalue_{h+1}^k}(s_h^k,a_h^k)\big)^{\top}\bphi_{\vvalue_{h+1}^k}(s_h^k,a_h^k) $\label{algorithm2:line12}
    \STATE  $\bbb_{h}^{k+1}=\bbb_{h}^{k}+\vvalue_{h+1}^k(s_{h+1}^k)\cdot \bphi_{\vvalue_{h+1}^k}(s_h^k,a_h^k)$\label{algorithm2:line13}
    \ENDFOR
\ENDFOR
    \end{algorithmic}
\end{algorithm}

\subsection{Regret Analysis}
In this subsection, we propose our regret analysis for UCRL-VTR. For simplicity, we denote $T = KH$, which is the total number of steps.
\begin{theorem}\label{thm:2}
Suppose Assumption \ref{assumption:gap} and Assumption \ref{assumption-linear-mixture} hold. If we set $\lambda=H^2d$ and $\beta_k=4C_{\theta}H\sqrt{d\log(1+Hk)\log^2\big((k+1)^2H/\delta\big)}$ in Algorithm \ref{algorithm2}, then with probability at least $1-2(K+1)H\log(H/\text{gap}_{\min})\delta-\log T\delta$, the  regret for Algorithm \ref{algorithm2} in first $T$ steps is upper bounded by 
\begin{align} 
    \text{Regret}(K)\leq \frac{4097C^2_{\btheta}d^2H^5 \log^3(2dT/\delta)}{\text{gap}_{\min}}\iota+\frac{16H^2\log \delta}{3},\notag
\end{align}
where $\iota$ is defined as follows: 
\begin{align}
     \iota = \log \bigg(\frac{512C^2_{\btheta}d^2H^4 \log^3(2dT/\delta)}{\text{gap}^2_{\min}}\bigg).\notag
\end{align}
\end{theorem}
\begin{remark} If we set the $\delta$ in Theorem \ref{thm:2} as $\delta=1/(2K(K+1)H^3)$ and define the high probability event $\Omega$ as: $\{\text{Theorem } \ref{thm:2} \text{ holds}\}$. Then, for the expected regret, we have
\begin{align} 
    \EE\big[\text{Regret}(K)\big]
    &\leq\EE\big[\text{Regret}(K)|\Omega\big]\Pr[\Omega]+T\Pr[\bar{\Omega}]\notag\\
    &\leq \frac{4097C^2_{\btheta}d^2H^5 \log^3(2dT/\delta)}{\text{gap}_{\min}}\iota +\frac{16H^2\log \delta}{3} +T\Pr[\bar{\Omega}]
    \notag\\
&=\tilde O(d^2H^5/\text{gap}_{\min}\log T)\notag.
\end{align}
\end{remark}

The regret bound in Theorem \ref{thm:2} depends on $\text{gap}_{\min}$ inversely. It is independent of the size of the state, action space $\cS,\cA$, and is logarithmic in the number of steps $T$, similar to that of Theorem~\ref{thm:1}. This suggests that model-based RL with linear function approximation also enjoys a $\log T$-type regret considering the problem structure.

Similar to the model-free setting, the following theorem gives a lower bound of the regret for any algorithm learning linear mixture MDPs.
\begin{theorem}\label{thm:4}
Suppose $\text{gap}_{\min}\leq 1/(3dH), H\ge 3,$ then for any algorithm,  there exist a linear mixture MDP such that $C_{\theta}=2$ and the lower bounded of the expected regret is bounded by
\begin{align}
    \EE\big[\text{Regret(K)}\big]\ge\Omega\bigg(\frac{Hd}{\text{gap}_{\min}}\bigg).\notag
\end{align}
\end{theorem}

\section{Proof of the Main Results}\label{section: proof-of-main}

In this section, we give a proof outline of Theorem \ref{thm:1}, along with the proofs of the key technical lemmas.

\subsection{Proof of Theorem \ref{thm:1}}\label{section:second-1}
 The proof can be divided into three main steps.

\noindent\textbf{Step 1: Regret decomposition} 
 
Our goal is to upper bound the total regret $\text{Regret}(K)$. Following the regret decomposition procedure proposed in \citet{simchowitz2019non, yang2020q}, for a given policy $\pi$, we rewrite the sub-optimality $\vvalue_h^*(s_h)-\vvalue_h^{\pi_k}(s_h)$ as follows:
\begin{align}
    \vvalue_h^*(s_h)-\vvalue_h^{\pi}(s_h) &=\big(\vvalue_h^*(s_h)-\qvalue_h^*(s_h,a_h)\big)+\big(\qvalue_h^*(s_h,a_h)-\vvalue_h^{\pi_k}(s_h)\big)\notag\\
    & =\text{gap}_h(s_h,a_h)+\EE_{s'\sim \PP_h(\cdot| s_h,a_h)}\big[\vvalue_{h+1}^*(s')-\vvalue_{h+1}^{\pi_k}(s')\big],\label{eq:seperate}
\end{align}
where $a_h=\pi(s_h,h)$ and  $\text{gap}_h(s,a)=\vvalue_h^*(s)-\qvalue_h^*(s,a)$. Taking expectation on both sides of \eqref{eq:seperate} with respect to the randomness of state-transition and taking summation over all $h\in[H]$, for any policy $\pi$ and initial state $s_1^k$, we have
\begin{align}
    \vvalue_1^*(s_1^k)-\vvalue_1^{\pi}(s_1^k)=\EE \bigg[\sum_{h=1}^H \text{gap}_h(s_h,a_h)\bigg],\label{eq:111}
\end{align}
where $s_1=s_1^k$ and for each $h\in[H]$, $a_h=\pi(s_h,h)$,$s_{h+1} \sim \PP_h(\cdot|s_h,a_h)$. Taking summation of \eqref{eq:111} over all $k\in[K]$ with $\pi=\pi_k$, we have
\begin{align}
    \EE \big[\text{Regret}(K)\big]=\EE \bigg[\sum_{k=1}^K\sum_{h=1}^H \text{gap}_h(s_h^k,a_h^k)\bigg]\label{eq:112}.
\end{align}
Furthermore, we have
\begin{lemma}\label{lemma: concentration}
For each MDP $\cM(\cS, \cA, H, \reward_h, \PP_h)$ and any $\tau>0$, with probability at least $1-me^{-\tau}$, we have
\begin{align}
    \text{Regret}(K)\leq 2\sum_{k=1}^K\sum_{h=1}^H\text{gap}_h(s_h^k,a_h^k)+\frac{16H^2\tau}{3}+2.\notag
\end{align}
where $m=\lceil \log T \rceil$.
\end{lemma}
\noindent Lemma \ref{lemma: concentration} and \eqref{eq:111} suggest that the total (expected) regret can be represented as a summation of $\text{gap}_h(s_h^k,a_h^k)$ over time step $h$ and episode $k$. Therefore, to bound the total regret, it suffices to bound each $\text{gap}_h(s_h^k,a_h^k)$ separately, which leads to our next proof step.

\noindent\textbf{Step 2: Bound the number of sub-optimalities} 
 
 Recall 
 the range of sub-optimality gap $\text{gap}_h(s_h^k,a_h^k)$ is $[\text{gap}_{\min},H]$. Therefore, to bound the summation of $\text{gap}_h(s_h^k,a_h^k)$, it suffices to divide the range $[\text{gap}_{\min},H]$ into several intervals and count the number of $\text{gap}_h(s_h^k,a_h^k)$ falling into each interval. Such a division is also used in \citet{yang2020q} which is similar to the ``peeling technique" widely used in local Rademacher complexity analysis \citep{bartlett2005local}. 
 Formally speaking, we divide the interval  $[\text{gap}_{\min},H]$ to $N=\big\lceil \log(H/ \text{gap}_{\min})\big\rceil$ intervals $\big[2^{i-1}\text{gap}_{\min},2^i\text{gap}_{\min}\big) \big(i\in[N]\big)$. Therefore, for each $\text{gap}_h(s_h^k,a_h^k)$ falling into $\big[2^{i-1}\text{gap}_{\min},2^i\text{gap}_{\min}\big) \big(i\in[N]\big)$, it can be upper bounded by $2^i\text{gap}_{\min}$. Meanwhile, we have the following inequality by considering $\vvalue_h^*(s_h^k)-\qvalue_h^{\pi_k}(s_h^k,a_h^k)$, which is the upper bound of $\text{gap}_h(s_h^k,a_h^k)$:
\begin{align}
  \vvalue_h^*(s_h^k)-\qvalue_h^{\pi_k}(s_h^k,a_h^k)\ge\text{gap}_h(s_h^k,a_h^k)\ge 2^{i-1}\text{gap}_{\min}\notag,
\end{align}
which suggests that to count how many $\text{gap}_h(s_h^k,a_h^k)$ belong to the interval $\big[2^{i-1}\text{gap}_{\min},2^i\text{gap}_{\min}\big) \big(i\in[N]\big)$, we only need to count the number of sub-optimalities  $\vvalue_h^*(s_h^k)-\qvalue_h^{\pi_k}(s_h^k,a_h^k)$ belonging to the interval. 
The following lemma is our main technical lemma. It is inspired by \citet{jin2019provably}, and it shows that the number of sub-optimalities can indeed be upper bounded. 
\begin{lemma}\label{lemma:gap-number}
There exist a constant $C$ such that, for any  $h\in[H]$, $n\in N$, with probability at least $1-(K+1)\delta$, we have
\begin{align}
    &\sum_{k=1}^K  \ind\big[\vvalue_h^*(s_h^k)-\qvalue_h^{\pi_k}(s_h^k,a_h^k)\ge 2^n\text{gap}_{\min}\big] \leq \frac{Cd^3H^4 \log(2dT/\delta)}{4^n\text{gap}^2_{\min}} \log \bigg(\frac{Cd^3H^4 \log(2dT/\delta)}{4^n\text{gap}^2_{\min}}\bigg).\notag
\end{align}
\end{lemma}

\noindent\textbf{Step 3: Summation of total error} 
 
Lemma \ref{lemma:gap-number} gives an upper bound on the number of $\text{gap}_h(s_h^k,a_h^k)$ in each interval $\big[2^{i-1}\text{gap}_{\min},2^i\text{gap}_{\min}\big)$. 
We further give the following upper bound for the $\text{gap}_h(s_h^k,a_h^k)$ within each interval:
\begin{align}
    \sum_{\text{gap}_h(s_h^k,a_h^k) \in \big[2^{i-1}\text{gap}_{\min},2^i\text{gap}_{\min}\big)}
    \text{gap}_h(s_h^k,a_h^k)
    & \leq \sum_{k=1}^K 2^i\text{gap}_{\min} \ind\Big[\text{gap}_h(s_h^k,a_h^k) \in \big[2^{i-1}\text{gap}_{\min},2^i\text{gap}_{\min}\big)\Big]\notag\\
    &\leq \sum_{k=1}^K 2^i\text{gap}_{\min} \ind\big[\vvalue_h^*(s_h^k)-\qvalue_h^{\pi_k}(s_h^k,a_h^k)\ge 2^{i-1}\text{gap}_{\min}\big].\notag
\end{align}
Thus, by using the upper bound on the number of $\text{gap}_h(s_h^k,a_h^k)$ in Lemma \ref{lemma:gap-number}, we have the following lemma:
\begin{lemma}\label{lemma:gap-sum}
There exist a constant $C$ such that, for $h\in[H]$, with probability at least $1-2(K+1)\log(H/\text{gap}_{\min})\delta$, we have
\begin{align}
    &\sum_{k=1}^K  \big(\vvalue_h^*(s_h^k)-\qvalue_h^{*}(s_h^k,a_h^k)\big) \leq \frac{4Cd^3H^4 \log(2dT/\delta)}{\text{gap}_{\min}}\log \bigg(\frac{Cd^3H^4 \log(2dT/\delta)}{\text{gap}^2_{\min}}\bigg).\notag
\end{align}
\end{lemma}
Lemma \ref{lemma:gap-sum} suggests that with high probability, the summation of $\text{gap}_h(s_h^k,a_h^k)$ over episode $k$ at step $h$ is logarithmic in the number of steps $ T =KH$ and its dependency in $\text{gap}_{\min}$ is $1 / \text{gap}_{\min}$. This leads to our final proof of our main theorem. 

\begin{proof}[Proof of Theorem \ref{thm:1}]
\noindent We define the high probability event $\Omega$ as follows.
\begin{align}
    \Omega &= \{\text{Lemma \ref{lemma:gap-sum} holds for all }  h\in[H], \text{ and Lemma \ref{lemma: concentration} holds on for }          \tau=\lceil\log (1/\delta)\rceil\}.\notag
\end{align}
According to Lemma \ref{lemma:gap-sum} and Lemma \ref{lemma: concentration}, we have $\Pr[\Omega]\ge 1-2(K+1)H\log(H/\text{gap}_{\min})\delta-\delta\log T  $.
Given the event $\Omega$, we have
\begin{align}
   \text{Regret}(K)
    &\leq 2\sum_{k=1}^K\sum_{h=1}^H\text{gap}_h(s_h^k,a_h^k)+\frac{16H^2\log \delta}{3}+2 \notag\\
    &= 2\sum_{k=1}^K\sum_{h=1}^H\vvalue_h^*(s_h^k)-\qvalue_h^{*}(s_h^k,a_h^k)+\frac{16H^2\log \delta}{3}+2\notag\\
    &\leq  \frac{9Cd^3H^5 \log(2dHK/\delta)}{\text{gap}_{\min}}  \log \bigg(\frac{Cd^3H^4 \log(2dHK/\delta)}{\text{gap}^2_{\min}}\bigg) +\frac{16H^2\log \delta}{3},\notag
\end{align}
where the first inequality holds due to Lemma \ref{lemma: concentration} and the last inequality holds due to Lemma \ref{lemma:gap-sum}. Thus, we complete the proof.
\end{proof}

\subsection{Proof of the Key Technical Lemma}\label{section:third}
In this subsection, we propose the proof to the main technical lemma, Lemma \ref{lemma:gap-number}. Our proof follows the idea of error decomposition proposed in \citet{dong2019q, yang2020q}, that is, to upper bound the summation of sub-optamalities by considering their summation of the exploration bonuses. The key difference between our proof and that of \citet{dong2019q, yang2020q} is the choice of exploration bonus. \citet{dong2019q, yang2020q} considered the tabular MDP setting and adapted a $1/\sqrt{n}$-type bonus term, while we consider the linear function approximation setting and adapt a linear bandit-style exploration bonus \citep{dani2008stochastic, abbasi2011improved, li2010contextual} as suggested in Line \ref{algorithm:line3}. The following lemmas guarantee that our constructed $Q_h^k$ is indeed the UCB of the optimal action-value function:
\begin{lemma}[Lemma B.4 in \citealt{jin2019provably}]\label{lemma:transition}
With probability at least $1-\delta$, for any policy $\pi$ and all $s\in \cS, a\in \cA, h\in[H], k\in[K],$ we have
\begin{align}
    \big\la\bphi(s,a),\wb_h^k\big\ra-\qvalue_h^{\pi}(s,a)=\big[\PP_h(\vvalue_{h+1}^{k} -\vvalue_{h+1}^{\pi})\big](s,a)+\Delta,\notag
\end{align}
where $|\Delta|\leq \beta \sqrt{\bphi(s,a)^{\top}(\Lambda_h^k)^{-1}\bphi(s,a)}$
\end{lemma}
\begin{lemma}[Lemma B.5 in \citealt{jin2019provably}]\label{lemma:UCB}
With probability at least $1-\delta$, for all $s\in \cS, a\in \cA, h\in[H], k\in[K],$  we have
\begin{align}
    \qvalue_h^k(s,a)\ge \qvalue_h^*(s,a).\notag
\end{align}
\end{lemma}

We also need the following technical lemma, which gives us a slightly stronger upper bound for the summation of exploration bonuses:
\begin{lemma}\label{lemma:sum}
 For any subset $C=\{c_1,..,c_k\} \subseteq [K]$ and any $h\in[H]$, we have
\begin{align}
    \sum_{i=1}^{k} (\bphi_h^{c_i})^{\top}(\Lambda_h^{c_i})^{-1}\bphi_h^{c_i}\leq 2d\log\bigg(\frac{\lambda+k}{\lambda}\bigg),\notag
\end{align}
where $\bphi_h^{c_i}$ is the abbreviation of $\bphi_h^{c_i}(s_h^{c_i},a_h^{c_i})$.
\end{lemma}
With the lemmas above, we begin to prove Lemma \ref{lemma:gap-number}. 
\begin{proof}[Proof of Lemma \ref{lemma:gap-number}]
\noindent We fix $h$ in this proof. Let $k_0=0$, and for $i\in[N]$, we denote $k_i$ as the minimum index of the episode where the sub-optimality at step $h$ is no less than $2^n\text{gap}_{\text{min}}$:
\begin{align}
    k_i&=\min \big\{k: k>k_{i-1}, \vvalue_h^*(s_h^k)-\qvalue_h^{\pi_k}(s_h^k,a_h^k)\ge 2^n\text{gap}_{\min}\big\}.\label{eq:ti}
\end{align}
For simplicity, we denote by $K'$ the number of episodes such that the sub-optimality of this episode at step $h$ is no less than $2^n\text{gap}_{\text{min}}$. Formally speaking, we have
\begin{align}
    K'=\sum_{k=1}^K  \ind\big[\vvalue_h^*(s_h^k)-\qvalue_h^{\pi_k}(s_h^k,a_h^k)\ge 2^n\text{gap}_{\min}\big].\notag
\end{align}
From now we only consider the episodes whose sub-optimality is no less than $2^n\text{gap}_{\text{min}}$. We first lower bound the summation of difference between the estimated action-value function $Q_h^{k_i}$ and the action-value function induced by the policy $\pi_{k_i}$, which can be represented as follows:
\begin{align}
    \sum_{i=1}^{K'}\big(\qvalue_h^{k_i}(s_h^{k_i},a_h^{k_i})-\qvalue_h^{\pi_{k_i}}(s_h^{k_i},a_h^{k_i})\big)
    & \ge \sum_{i=1}^{K'}\Big(\qvalue_h^{k_i}\big(s_h^{k_i},\pi_h^*(s_h^{k_i},h)\big)-\qvalue_h^{\pi_{k_i}}(s_h^{k_i},a_h^{k_i})\Big)\notag\\
    & \ge \sum_{i=1}^{K'}\Big(\qvalue_h^*\big(s_h^{k_i},\pi_h^*(s_h^{k_i},h)\big)-\qvalue_h^{\pi_{k_i}}(s_h^{k_i},a_h^{k_i})\Big)\notag\\
    & = \sum_{i=1}^{K'}\big(\vvalue_h^*(s_h^{k_i})-\qvalue_h^{\pi_{k_i}}(s_h^{k_i},a_h^{k_i})\big)\notag\\
    & \ge 2^n\text{gap}_{\min} K',\label{eq:lower-Regret}
\end{align}
where the first inequality holds due to the definition of policy $\pi_{k_i}$, the second inequality holds due to Lemma \ref{lemma:UCB} and the last inequality holds due to the definition of $k_i$ in \eqref{eq:ti}. On the other hand, we upper bound $\sum_{i=1}^{K'}\big(\qvalue_h^{k_i}(s_h^{k_i},a_h^{k_i})-\qvalue_h^{\pi_{k_i}}(s_h^{k_i},a_h^{k_i})\big)$ as follows. For any $h'\in [H], k\in [K]$, we have
\begin{align}
    &\qvalue_{h'}^k(s_{h'}^k,a_{h'}^k)-\qvalue_{h'}^{\pi_k}(s_{h'}^k,a_{h'}^k)\notag\\
    & =\big\la \bphi(s_{h'}^k,a_{h'}^k),\wb_{h'}^k \big\ra -\qvalue_{h'}^{\pi_k}(s_{h'}^k,a_{h'}^k)+\beta \sqrt{\bphi(s_{h'}^k,a_{h'}^k)^{\top}(\Lambda_{h'}^k)^{-1}\bphi(s_{h'}^k,a_{h'}^k)} \notag\\
    & \leq \big[\PP_h(\vvalue_{h'+1}^{k} -\vvalue_{h'+1}^{\pi_{k}})\big](s_{h'}^k,a_{h'}^k) +2\beta \sqrt{\bphi(s_{h'}^k,a_{h'}^k)^{\top}(\Lambda_{h'}^k)^{-1}\bphi(s_{h'}^k,a_{h'}^k)}\notag\\
    & = \vvalue_{h'+1}^{k}(s_{h'+1}^k) -\vvalue_{h'+1}^{\pi_{k}}(s_{h'+1}^k)+\epsilon_{h'}^k +2\beta \sqrt{\bphi(s_{h'}^k,a_{h'}^k)^{\top}(\Lambda_{h'}^k)^{-1}\bphi(s_{h'}^k,a_{h'}^k)}\notag\\
    &= \qvalue_{h'+1}^k(s_{h'+1}^k,a_{h'+1}^k)-\qvalue_{h'+1}^{\pi_k}(s_{h'+1}^k,a_{h'+1}^k)+\epsilon_{h'}^k  +2\beta \sqrt{\bphi(s_{h'}^k,a_{h'}^k)^{\top}(\Lambda_{h'}^k)^{-1}\bphi(s_{h'}^k,a_{h'}^k)}\label{eq:tele},
\end{align}
where \begin{align}
    \epsilon_{h'}^{k}&= \big[\PP_h(\vvalue_{h'+1}^{k} -\vvalue_{h'+1}^{\pi_{k}})\big](s_{h'}^k,a_{h'}^k)  -\big(\vvalue_{h'+1}^{k}(s_{h'+1}^k) -\vvalue_{h'+1}^{\pi_{k}}(s_{h'+1}^k)\big),\notag
\end{align} and the inequality holds due to Lemma \ref{lemma:transition}. Taking summation for \eqref{eq:tele} over all $k_i$ and $h\leq h'\leq H$, we have
\begin{align}
    &\sum_{i=1}^{K'}\big(\qvalue_h^{k_i}(s_h^{k_i},a_h^{k_i})-\qvalue_h^{\pi_{k_i}}(s_h^{k_i},a_h^{k_i})\big)- \underbrace{\sum_{i=1}^{K'}\sum_{h'=h}^{H}\epsilon_{h'}^{k_i}}_{I_1}  
     \leq \underbrace{\sum_{i=1}^{K'}\sum_{h'=h}^{H}2\beta \sqrt{\bphi(s_{h'}^{k_i},a_{h'}^{k_i})^{\top}(\Lambda_{h'}^{k_i})^{-1}\bphi(s_{h'}^{k_i},a_{h'}^{k_i})}}_{I_2} .\label{eq:upper-Regret}
\end{align}
It therefore suffices to bound $I_1$ and $I_2$ separately.  
For $I_1$, by Lemma \ref{lemma:azuma}, for each episode $k\in [K]$, with probability at least $1-\delta$, we have
\begin{align}
    &\sum_{i=1}^k \sum_{j=h}^{H}\Big(\big[\PP_j(\vvalue^{k_i}_{j+1}-\vvalue^{\pi_{k_i}}_{j+1})\big](s^{k_i}_{j},a^{k_i}_{j}) - \big(\vvalue^{k_i}_{j+1}(s^{k_i}_{j+1})-\vvalue^{\pi_{k_i}}_{j+1}(s^{k_i}_{j+1})\big)\Big) \leq \sqrt{2kH^2\log(1/\delta)},\notag
\end{align}
where we use the fact that $\big[\PP_j(\vvalue^{k_i}_{j+1}-\vvalue^{\pi_{k_i}}_{j+1})\big](s^{k_i}_{j},a^{k_i}_{j}) -\big(\vvalue^{k_i}_{j+1}(s^{k_i}_{j+1})-\vvalue^{\pi_{k_i}}_{j+1}(s^{k_i}_{j+1})\big)$ forms a martingale difference sequence. Taking a union bound for all $k\in[K]$ gives that, with probability at least $1-K\delta$, 
\begin{align}
        \sum_{i=1}^{K'} \sum_{j=h}^{H}\big[\PP_j(\vvalue^{k_i}_{j+1}-\vvalue^{\pi_{k_i}}_{j+1})\big](s^{k_i}_{j},a^{k_i}_{j}) -\sum_{i=1}^k \sum_{j=h}^{H}\big(\vvalue^{k_i}_{j+1}(s^{k_i}_{j+1})-\vvalue^{\pi_{k_i}}_{j+1}(s^{k_i}_{j+1})\big)\leq \sqrt{2K'H^2\log(1/\delta)}.\label{eq:I_2}
\end{align}
For $I_2$, we have
\begin{align}
    I_1&=\sum_{i=1}^{K'}\sum_{h'=h}^{H}2\beta \sqrt{\bphi(s_{h'}^{k_i},a_{h'}^{k_i})^{\top}(\Lambda_{h'}^{k_i})^{-1}\bphi(s_{h'}^{k_i},a_{h'}^{k_i})}\notag\\
    &\leq 2\beta\sqrt{K'}\sum_{h'=h}^{H} \sqrt{\sum_{i=1}^{K'}\bphi(s_{h'}^{k_i},a_{h'}^{k_i})^{\top}(\Lambda_{h'}^{k_i})^{-1}\bphi(s_{h'}^{k_i},a_{h'}^{k_i}) }\notag\\
    &\leq 2H\beta \sqrt{K'}\sqrt{2d\log(K'+1)},\label{eq:I_1}
\end{align}
where the first inequality holds due to Cauchy-Schwarz inequality and the second inequality holds due to Lemma~\ref{lemma:sum}. 

Substituting \eqref{eq:I_1} and \eqref{eq:I_2} into \eqref{eq:upper-Regret}, we obtain that with probability at least $1-(K+1)\delta$, 
\begin{align}
     &\sum_{i=1}^{K'}\big(\qvalue_h^{k_i}(s_h^{k_i},a_h^{k_i})-\qvalue_h^{\pi_{k_i}}(s_h^{k_i},a_h^{k_i})\big)\leq  \sqrt{2K'H^2\log(1/\delta)}+ 2H\beta \sqrt{K'}\sqrt{2d\log(K'+1)}.\label{eq:Upper}
\end{align}
By now, we have obtained both the lower and upper bounds for $\sum_{i=1}^{K'}\big(\qvalue_h^{k_i}(s_h^{k_i},a_h^{k_i})-\qvalue_h^{\pi_{k_i}}(s_h^{k_i},a_h^{k_i})\big)$ from \eqref{eq:lower-Regret} and \eqref{eq:Upper}. Finally, combining \eqref{eq:lower-Regret} and \eqref{eq:Upper}, we can derive the following constraint on $K'$:
\begin{align}
    2^n\text{gap}_{\min}K'&\leq \sqrt{2K'H^2\log(1/\delta)} +2H\beta \sqrt{2K'd\log(K'+1)}.\label{eq:ppp}
\end{align}
Solving out $K'$ from \eqref{eq:ppp}, we conclude that there exists a constant $C$ such that
\begin{align}
  K'&\leq \frac{Cd^3H^4 \log(2dHK/\delta)}{4^n\text{gap}^2_{\min}} \times\log \bigg(\frac{Cd^3H^4 \log(2dHK/\delta)}{4^n\text{gap}^2_{\min}}\bigg), \notag
\end{align}
which ends our proof.
\end{proof}


\section{Conclusion}
In this paper, we analyze the RL algorithms with function approximation by considering a specific problem-dependent quantity $\text{gap}_{\text{min}}$. We show that two existing algorithms LSVI-UCB and UCRL-VTR attain $\log T$-type regret instead of $\sqrt{T}$-type regret under their corresponding linear function approximation assumptions. It remains unknown whether the dependence of the length of the episode $H$ and dimension $d$ is optimal or not, and we leave it as future work.

\appendix


\section{Additional Proofs of the Main Results}\label{section:jia}
\subsection{Proof of Theorem \ref{thm:3}}
In this section, we give a proof of Theorem \ref{thm:3} and the lower bound is based on previous work \citep{zhou2020nearly}.
\begin{proof}[Proof of Theorem \ref{thm:3}]
To prove the lower bound, we construct a series of hard instances based on the hard-to-learn MDPs introduced by \citet{zhou2020provably,zhou2020nearly}. To be more specific, the state space $\cS$ consists of state $s_1,..,s_{H+2}$, where $s_{H+1}$ and $s_{H+2}$ are absorbing states. The action space $\cA=\{-1,1\}^{d-1}$ consists of $2^{d-1}$ different actions. For each action $\ab\in \cA$ and step $h\in[H]$, the reward function satisfies that $\reward_h(s_{h'},\ba)=0(1\leq h'\leq H+1)$ and $\reward_h(s_{H+2},\ba)=1$. For the transition probability function $\PP_h$, $s_{H+1}$ and $s_{H+2}$ are absorbing states, which will always stay at the same state, and for other state $s_{h'}(1\leq h'\leq H)$, we have
\begin{align}
    &\PP_h(s_{h'+1}|s_{h'},\ab)=1-\delta-\la\bmu_{h},\ab\ra,\notag\\
    &\PP_h(s_{H+2}|s_{h'},\ab)=\delta+\la\bmu_h,\ab\ra\notag.
\end{align}
where each $\bmu_h\in \{-\Delta,\Delta\}^{d}$ and $\delta=1/H$.
Remark 5.8 \citep{zhou2020nearly} shows that these MDPs can be illustrated as a linear MDP with certain parameters $\bphi(s,a)$,$\btheta_h(s')$ and $\bmu_h$. Furthermore, for each step $h\in[H-1]$ and non-absorbing state $s_{h}'(1\leq h'\leq H)$, the optimal policy $\pi^*$ is picking the action $\ab^*=\bmu_{h}/\Delta$ and for each non-optimal action $\ab\ne \ab^*$, we have
\begin{align}
    \qvalue_{h}^*(s_{h'},\ab^*)-\qvalue_{h}^*(s_{h'},a)
    &=\big(\PP_h(s_{h'+1}|s_{h'},\ab^*)-\PP_h(s_{h'+1}|s_{h'},\ab)\big)\vvalue_{h+1}^*(s_{h'+1})\notag\\
    &\qquad +\big(\PP_h(s_{H+2}|s_{h'},\ab^*)-\PP_h(s_{H+2}|s_{h'},\ab)\big)\vvalue_{h+1}^*(s_{H+2})\notag\\
    &=\la\bmu_h,\ab^*-\ab\ra\big(\vvalue_{h+1}^*(s_{H+2})-\vvalue_{h+1}^*(s_{h'+1})\big)\notag\\
    &\ge 2\Delta \big(\vvalue_{h+1}^*(s_{H+2})-\vvalue_{h+1}^*(s_{h'+1})\big)\notag\\
    &\ge 2\Delta,\notag
\end{align}
where the first inequality holds due to the definition of optimal action $\ab^*$, the second inequality holds due to $\vvalue_{h+1}^*(s_{H+2})=H-h$ and $\vvalue_{h+1}^*(s_{h'+1})=\reward_{h+1}\big(s_{h'+1},\pi^*(s_{h'+1})\big)+[\PP_{h+1}\vvalue_{h+2}]\big(s_{h'+1},\pi^*(s_{h'+1})\big)\leq H-h-1$. Furthermore, for step $h=H$ or absorbing states $s_{H+1},s_{H+2}$, the reward and future reward will not change whatever the action $\ab$ is chosen, which means the value function $\qvalue_{h}(s,\ab)$ remains same for all action $\ab$. Thus, the minimal sub-optimality gap of these hard-to-learn MDPs is $\text{gap}_{\min}=2\Delta$.

For the regret of these hard-to-learn MDPs, Theorem 5.6 \citep{zhou2020nearly} shows that for any algorithm, if $H\ge 3, 0\leq \delta\leq 1/3 $, $3(d-1)\Delta\leq \delta $ and $K\ge d^2/(2\delta)$, then there exist a parameter $\bmu^*=\{\bmu_1^*,..,\bmu_H^*\}$ and the regret of the corresponding hard-to-learn MDP is bounded by
\begin{align}
    \text{Regret(K)}\geq \frac{H^2}{20}\Delta d\big(K-\sqrt{2}K^{3/2}\Delta/\sqrt{\delta}\big).\notag
\end{align}
Furthermore, we can choose the number of episodes as $K=\delta/(32\Delta^2)$ and under this case, the regret is bounded by
\begin{align}
    \text{Regret(K)}\ge \frac{H^2}{20}\frac{(d-1)\delta}{64\Delta}=\Omega\Big(\frac{Hd}{\text{gap}_{\min}}\Big).\notag
\end{align}
Therefore, we finish the proof of Theorem \ref{thm:3}.
\end{proof}

\subsection{Proof of Theorem \ref{thm:2}}

\begin{lemma}\label{lemma: confidence-set}
With probability at least $1-\delta$, for all $h\in[H]$,$k\in[K]$, we have $\btheta_h\in B_{h,k}$, where the confidence set $B_k$ is denoted by
\begin{align}
    B_{h,k}=\big\{\btheta|(\btheta-\btheta_{k,h})^\top \bSigma_h^k (\btheta-\btheta_{k,h})\leq \beta^2_k\big\}\notag.
\end{align}
\end{lemma}

\begin{lemma}\label{lemma:jia-transition}
With probability at least $1-\delta$, for  all $s\in \cS, a\in \cA, h\in[H], k\in[K],$ we have
\begin{align}
    \qvalue_h^k(s,a)-\qvalue_h^{\pi_k}(s,a)\leq\big[\PP_h(\vvalue_{h+1}^{k} -\vvalue_{h+1}^{\pi_k})\big](s,a)+2\beta_k \sqrt{\bphi(s,a)^{\top}(\Lambda_1^k)^{-1}\bphi(s,a)}.\notag
\end{align}
Furthermore, we have $\qvalue_h^k(s,a)\ge \qvalue_h^*(s,a).$
\end{lemma}

\begin{lemma}\label{lemma:jia-sum}
 For any subset $C=\{c_1,..,c_k\} \subseteq [K]$ and any $h\in[H]$, we have
\begin{align}
    \sum_{i=1}^{k} (\bphi_h^{c_i})^{\top}(\bSigma_h^{c_i})^{-1}\bphi_h^{c_i}\leq 2d\log(1+k),\notag
\end{align}
where $\bphi_h^{c_i}$ is the abbreviation of $\bphi_{\vvalue_{h+1}^{c_i}}(s_h^{c_i},a_h^{c_i})$.
\end{lemma}

\begin{lemma}\label{lemma:jia-gap-number}
For each $h\in[H]$, $n\in N$, with probability at least $1-(K+1)\delta$, we have
\begin{align}
    &\sum_{k=1}^K  \ind\big[\vvalue_h^*(s_h^k)-\qvalue_h^{\pi_k}(s_h^k,a_h^k)\ge 2^n\text{gap}_{\min}\big] \leq \frac{512C^2_{\btheta}d^2H^4 \log^3(2dT/\delta)}{4^n\text{gap}^2_{\min}} \log \bigg(\frac{512C^2_{\btheta}d^2H^4 \log^3(2dT/\delta)}{4^n\text{gap}^2_{\min}}\bigg).\notag
\end{align}
\end{lemma}

\begin{lemma}\label{lemma:jia-gap-sum}
For each $h\in[H]$, with probability at least $1-2(K+1)\log(H/\text{gap}_{\min})\delta$, we have
\begin{align}
    \sum_{k=1}^K  \big(\vvalue_h^*(s_h^k)-\qvalue_h^{*}(s_h^k,a_h^k)\big)\notag\leq \frac{2048C^2_{\btheta}d^2H^4 \log^3(2dT/\delta)}{\text{gap}_{\min}} \log \bigg(\frac{512C^2_{\btheta}d^2H^4 \log^3(2dT/\delta)}{\text{gap}^2_{\min}}\bigg).\notag
\end{align}
\end{lemma}

\begin{proof}[Proof of Theorem \ref{thm:2}]

\noindent We define the high probability event $\Omega$ as follows:
\begin{align}
    \Omega = \{\text{Lemma \ref{lemma:jia-gap-sum} holds for all }  h\in[H] \text{ and Lemma \ref{lemma: concentration} holds on for }          \tau=\lceil\log (1/\delta)\rceil\}.\notag
\end{align}
According to Lemma \ref{lemma:jia-gap-sum} and Lemma \ref{lemma: concentration}, we have $\Pr[\Omega]\ge 1-2(K+1)H\log(H/\text{gap}_{\min})\delta-\delta\log T  $.
Given the event $\Omega$, we have
\begin{align}
       \text{Regret}(K)
    &\leq 2\sum_{k=1}^K\sum_{h=1}^H\text{gap}_h(s_h^k,a_h^k)+\frac{16H^2\log \delta}{3}+2 \notag\\
    &= 2\sum_{k=1}^K\sum_{h=1}^H\vvalue_h^*(s_h^k)-\qvalue_h^{*}(s_h^k,a_h^k)+\frac{16H^2\log \delta}{3}+2\notag\\
    & \leq \frac{4097C^2_{\btheta}d^2H^4 \log(2dHK/\delta)}{\text{gap}_{\min}}\log \bigg(\frac{512C^2_{\btheta}d^2H^4 \log(2dHK/\delta)}{\text{gap}^2_{\min}}\bigg)+\frac{16H^2\log \delta}{3},\notag
\end{align}
where the first inequality holds due to Lemma \ref{lemma: concentration} and the last inequality holds due to Lemma \ref{lemma:jia-gap-sum}. Thus, we completes the proof.
\end{proof}

\subsection{Proof of Theorem \ref{thm:4}}
\begin{proof}[Proof of Theorem \ref{thm:4}]
To prove the lower bound, we use the same  hard-to-learn MDPs introduced in the proof of the Theorem \ref{thm:3}. By the same analyze, for these MDPs, the minimal sub-optimality gap is $\text{gap}_{\min}=2\Delta$ and Theorem 5.6 also shows that these MDPs can be illustrated as a linear mixture MDP with $C_{\theta}=2$.
For the regret of these hard-to-learn MDPs, Theorem 5.6 \citep{zhou2020nearly} shows that for any algorithm, if $H\ge 3, 0\leq \delta\leq 1/3 $, $3(d-1)\Delta\leq \delta $ and $K\ge d^2/(2\delta)$, then there exist a parameter $\bmu^*=\{\bmu_1^*,..,\bmu_H^*\}$ and the regret of the corresponding hard-to-learn MDP is bounded by
\begin{align}
    \text{Regret(K)}\geq \frac{H^2}{20}\Delta d\big(K-\sqrt{2}K^{3/2}\Delta/\sqrt{\delta}\big).\notag
\end{align}
Furthermore, we can choose the number of episodes as $K=\delta/(32\Delta^2)$ and under this case, the regret is bounded by
\begin{align}
    \text{Regret(K)}\ge \frac{H^2}{20}\frac{(d-1)\delta}{64\Delta}=\Omega\Big(\frac{Hd}{\text{gap}_{\min}}\Big).\notag
\end{align}
Therefore, we finish the proof of the Theorem \ref{thm:4}.
\end{proof}

\section{Proof of Lemmas in Section \ref{section: proof-of-main}}

\subsection{Proof of Lemma \ref{lemma: concentration}}

\begin{proof}[Proof of Lemma \ref{lemma: concentration}]
For a given policy $\pi$ and any state $s_h\in \cS$, we have
\begin{align}
    \vvalue_h^*(s_h)-\vvalue_h^{\pi}(s_h) &=\big(\vvalue_h^*(s_h)-\qvalue_h^*(s_h,a_h)\big)+\big(\qvalue_h^*(s_h,a_h)-\vvalue_h^{\pi_k}(s_h)\big)\notag\\
    & =\text{gap}_h(s_h,a_h)+\EE_{s'\sim \PP_h(\cdot| s_h,a_h)}\big[\vvalue_{h+1}^*(s')-\vvalue_{h+1}^{\pi_k}(s')\big],\label{eq:seperate1}
\end{align}
where $a_h=\pi(s_h,h)$ and  $\text{gap}_h(s,a)=\vvalue_h^*(s)-\qvalue_h^*(s,a)$. Taking expectation on both sides of \eqref{eq:seperate1} with respect to the randomness of state-transition and taking summation over all $h\in[H]$,  for any policy $\pi$ and initial state $s_1^k$, we have
\begin{align}
    \vvalue_1^*(s_1^k)-\vvalue_1^{\pi}(s_1^k)=\EE \bigg[\sum_{h=1}^H \text{gap}_h(s_h,a_h)\bigg],\label{eq:61}
\end{align}
where $s_1=s_1^k$ and $a_h=\pi(s_h,h)$, $s_{h+1} \sim \PP_h(\cdot|s_h,a_h)$. Now, We denote the filtration $\mathcal{F}_k$ contain all randomness before the episode $k$, then policy $\pi_k$ is deterministic with respect to the filtration $\mathcal{F}_k$. Thus, we have
\begin{align}
    \EE \bigg[\sum_{h=1}^H \text{gap}_h(s_h,a_h)|\mathcal{F}_k\bigg]=\vvalue_1^*(s_1^k)-\vvalue_1^{\pi_k}(s_1^k).\label{eq:62}
\end{align}
For simplicity, we denote the random variable $X_k=\sum_{h=1}^H \text{gap}_h(s_h^k,a_h^k)-\big(\vvalue_1^*(s_1^k)-\vvalue_1^{\pi_k}(s_1^k)\big)$, then $\{X_k\}_{k=1}^{K}$ is a martingale difference sequence with respect to the filtration $\mathcal{F}_k$ and for the random variable $X_k$, we have $|X_k|\leq H^2$. Furthermore, for the variance of the random variable $X_k$, we have
\begin{align}
\EE[X_k^2|\mathcal{F}_k]
&\leq\EE\Big[\big(X_k+\vvalue_1^*(s_1^k)-\vvalue_1^{\pi_k}(s_1^k)\big)^2|\mathcal{F}_k\Big]\notag\\
&=\EE\Big[\big(\sum_{h=1}^H \text{gap}_h(s_h^k,a_h^k)        \big)^2|\mathcal{F}_k\Big]\notag\\
&  \leq H^2\EE\Big[\sum_{h=1}^H \text{gap}_h(s_h^k,a_h^k)|\mathcal{F}_k\Big]\notag\\
&=H^2\big(\vvalue_1^*(s_1^k)-\vvalue_1^{\pi_k}(s_1^k)\big),
\end{align}
where the first inequality holds due to the fact that $\EE\Big[\big(X-\EE[X]\big)^2\Big]\leq \EE[X^2]$ and the second inequality holds due to $0\leq\text{gap}_h(s_h^k,a_h^k)\leq H$. Therefore, the total variance of the random variables $\{X_k\}_{k=1}^{K}$ is bounded by
\begin{align}
    V=\sum_{k=1}^K\EE[X_k^2|\mathcal{F}_k]\leq \sum_{k=1}^K H^2\big(\vvalue_1^*(s_1^k)-\vvalue_1^{\pi_k}(s_1^k)\big)=H^2\text{Regret}(K).\notag
\end{align}
However, the upper bound of total variance is a random variable and we cannot use Lemma \ref{lemma:freedman} directly. Therefore, we consider for two different cases and use peeling technique to deal with this problem.

\noindent\textbf{Case 1:} $\text{Regret}(K)\leq 1$: Since $\text{Regret}(K)\leq 1$, we have
\begin{align}
    \sum_{k=1}^K X_k\ge -\sum_{k=1}^K\big(\vvalue_1^*(s_1^k)-\vvalue_1^{\pi_k}(s_1^k)\big) =- \text{Regret}(K) \ge -1,
\end{align}
where the first inequality holds due to $\text{gap}_h(s_h^k,a_h^k)\ge 0$ and the second inequality holds due to the assumption of Case 1.

\noindent\textbf{Case 2:} $\text{Regret}(K)> 1$: For any $\tau>0$ and $m=\lceil \log T \rceil$, by peeling technique, we have
\begin{align}
    &\Pr\bigg[\sum_{k=1}^K X_k\leq -2\sqrt{H^2\tau\text{Regret}(K)}-\frac{2H^2\tau}{3},1<\text{Regret}(K)\bigg]\notag\\
    &=\Pr\bigg[\sum_{k=1}^K X_k\leq -2\sqrt{H^2\tau\text{Regret}(K)}-\frac{2H^2\tau}{3},1< \text{Regret}(K)\leq T,V\leq H^2 \text{Regret}(K)\bigg]\notag\\
    &\leq \sum_{i=1}^m
    \Pr\bigg[\sum_{k=1}^K X_k\leq -2\sqrt{H^2\tau\text{Regret}(K)}-\frac{2H^2\tau}{3},2^{i-1}<\text{Regret}(K)\leq 2^i,V\leq H^2 \text{Regret}(K)\bigg]\notag\\
    &\leq \sum_{i=1}^m
    \Pr\bigg[\sum_{k=1}^K X_k\leq -\sqrt{2^{i+1}H^2\tau}-\frac{2H^2\tau}{3},V\leq 2^iH^2\bigg]\notag\\
    &\leq  \sum_{i=1}^m e^{-\tau}\notag\\
    &=me^{-\tau},\notag
\end{align}
where the first inequality holds due to $m=\lceil \log T \rceil$, the second inequality holds due to peeling technique and the last inequality holds due to Lemma \ref{lemma:freedman}.
Combining these two different cases, with probability at least $1-me^{-\tau},$ we have
\begin{align}
    \sum_{k=1}^K X_k\ge - 2\sqrt{H^2\tau \text{Regret}(K)}-\frac{2H^2\tau}{3}-1,\notag
\end{align}
where $X_k=\sum_{h=1}^H \text{gap}_h(s_h^k,a_h^k)-\big(\vvalue_1^*(s_1^k)-\vvalue_1^{\pi_k}(s_1^k)\big)$ and $\text{Regret}(K)=\sum_{k=1}^K \big(\vvalue_1^*(s_1^k)-\vvalue_1^{\pi_k}(s_1^k)\big)$. Therefore, by the fact that  $x\leq a\sqrt{x}+b$ implies $x\leq a^2+2b$, with probability at least $1-me^{-\tau}$, we have
\begin{align}
    \text{Regret}(K)\leq 2\sum_{k=1}^K\sum_{h=1}^H\text{gap}_h(s_h^k,a_h^k)+\frac{16H^2\tau}{3} +2.\notag
\end{align}
Thus, we finish the proof of Lemma \ref{lemma: concentration}.
\end{proof}

\subsection{Proof of Lemma \ref{lemma:gap-sum}}
\begin{proof}[Proof of Lemma \ref{lemma:gap-sum}]
By the definition of $\text{gap}_{\min}$ in \eqref{definition-gap-min}, for each $h\in[H], k\in[K]$, we have $\vvalue_h^*(s_h^k)-\qvalue_h^{*}(s_h^k,a_h^k)=0$ or  $\vvalue_h^*(s_h^k)-\qvalue_h^{*}(s_h^k,a_h^k)\ge \text{gap}_{\min}$. Thus, we divide the interval       $[\text{gap}_{\min},H]$ to $N=\big\lceil \log(H/ \text{gap}_{\min})\big\rceil$ intervals: $\big[2^{i-1}\text{gap}_{\min},2^i\text{gap}_{\min}\big) \big(i\in[N]\big)$ and with probability at least $ 1-2(K+1)\log(H/\text{gap}_{\min})\delta$, we have
\begin{align}
    \sum_{k=1}^K  \big(\vvalue_h^*(s_h^k)-\qvalue_h^{*}(s_h^k,a_h^k)\big)& \leq \sum_{i=1}^N\sum_{k=1}^K \ind\big[2^{i}\text{gap}_{\min}\ge \vvalue_h^*(s_h^k)-\qvalue_h^*(s_h^k,a_h^k)\ge 2^{i-1}\text{gap}_{\min}\big] \times 2^i\text{gap}_{\min}\notag\\
    & \leq \sum_{i=1}^N\sum_{k=1}^K \ind\big[ \vvalue_h^*(s_h^k)-\qvalue_h^{\pi_k}(s_h^k,a_h^k)\ge 2^{i-1}\text{gap}_{\min}\big]\times 2^i\text{gap}_{\min}\notag\\
    & \leq \sum_{i=1}^N   \frac{4Cd^3H^4 \log(2dHK/\delta)}{2^{i}\text{gap}_{\min}}\times \log \bigg(\frac{Cd^3H^4 \log(2dHK/\delta)}{4^{i-1}\text{gap}^2_{\min}}\bigg)\notag\\
    & \leq \frac{4Cd^3H^4 \log(2dHK/\delta)}{\text{gap}_{\min}}\log \bigg(\frac{Cd^3H^4 \log(2dHK/\delta)}{\text{gap}^2_{\min}}\bigg)\notag,
\end{align}
where the first inequality holds due to $\vvalue_h^*(s_h^k)-\qvalue_h^{*}(s_h^k,a_h^k)\ge \text{gap}_{\min}$ or $\vvalue_h^*(s_h^k)-\qvalue_h^{*}(s_h^k,a_h^k)=0$, the second inequality holds on due to $\vvalue_h^*(s_h^k)-\qvalue_h^{\pi_k}(s_h^k,a_h^k)\ge \vvalue_h^*(s_h^k)-\qvalue_h^*(s_h^k,a_h^k)$, the third inequality holds due to due to Lemma \ref{lemma:gap-number}. Thus, we finish the proof of Lemma \ref{lemma:gap-sum}.
\end{proof}

\subsection{Proof of Lemma \ref{lemma:sum}}
\begin{proof}[Proof of Lemma \ref{lemma:sum}]
For simplicity, we denote 
\begin{align}
\Lambda'_i=\lambda \Ib+\sum_{j=1}^i(\bphi'_i)^{\top}\bphi'_i\notag,
\end{align}
where $\bphi'_i$ is the abbreviation of $\bphi_h^{c_i}(s_h^{c_i},a_h^{c_i})$. Thus, we have
\begin{align}
    \sum_{i=1}^{k} \big(\bphi_h^{c_i}(s_h^{c_i},a_h^{c_i})\big)^{\top}(\Lambda_h^{c_i})^{-1}\big(\bphi_h^{c_i}(s_h^{c_i},a_h^{c_i})\big)
    & \leq \sum_{i=1}^{k} (\bphi'_i)^{\top}(\Lambda'_{i-1})^{-1}\bphi'_i \leq 2\log \bigg[\frac{\det(\Lambda'_k)}{\det(\Lambda'_0)}\bigg] \leq 2d\log\bigg(\frac{\lambda+k}{\lambda}\bigg),\notag
\end{align}
where the first inequality holds due to $\Lambda'_{i-1}\preceq \Lambda_h^{c_i}$, the second inequality holds due to Lemma \ref{Lemma:abba} and the last inequality holds due to $\|\Lambda'_k\|=\|\lambda \Ib+\sum_{i=1}^k(\bphi'_k)^{\top}\bphi'_k\|\leq \lambda+k.$ Thus, we finish the proof of Lemma \ref{lemma:sum}.
\end{proof}

\section{Proof of Lemmas in Appendix \ref{section:jia}}
\subsection{Proof of Lemma \ref{lemma: confidence-set}}
\begin{proof}[Proof of Lemma \ref{lemma: confidence-set}]
For each step $h\in[H]$, by the definition of $\btheta_{k,h}$ in Algorithm \ref{algorithm2}, we have
\begin{align}
    \btheta_{k,h}=\bigg(\lambda \Ib+\sum_{i=1}^{k-1}\bphi_{{\vvalue}_{h+1}^i}(s_h^i,a_h^i)\bphi_{{\vvalue}_{h+1}^i}(s_h^i,a_h^i)^\top\bigg)^{-1}\bigg(\sum_{i=1}^{k-1} \bphi_{{\vvalue}_{h+1}^i}(s_h^i,a_h^i) \vvalue_{h+1}^i(s_{h+1}^{i})\bigg).
\end{align}
Furthermore, for the expectation of the value function $\vvalue_{h+1}^i(s_{h+1}^{i})$, we have
\begin{align}
    [\PP_h\vvalue_{h+1}^i](s_h^i,a_h^i) &= \sum_{s\in\cS}\PP_h(s'|s_h^i,a_h^i)\vvalue_{h+1}^i(s') \notag \\
    &= \sum_{s\in\cS}\la \bphi(s'|s_h^i,a_h^i), \btheta_h^*\ra\vvalue_{h+1}^i(s')\notag \\
    & = \Big\la\sum_{s\in\cS} \bphi(s'|s_h^i,a_h^i)\vvalue_{h+1}^i(s'), \btheta_h^*\Big\ra\notag \\
    & = \la \bphi_{{\vvalue}_{h+1}^i}(s_h^i,a_h^i), \btheta_h^*\ra,
\end{align}
which means the sequence $\big\{\vvalue_{h+1}^i(s_{h+1}^{i})-[\PP_h\vvalue_{h+1}^i](s_h^i,a_h^i)\big\}$ is a martingale sequence. Furthermore, for each random variable $\vvalue_{h+1}^i(s_{h+1}^{i})-[\PP_h\vvalue_{h+1}^i](s_h^i,a_h^i)$ in the martingale sequence, we have $0\leq \vvalue_{h+1}^i(s_{i+1})\leq H$. Thus,  $\vvalue_{h+1}^i(s_{h+1}^{i})-[\PP_h\vvalue_{h+1}^i](s_h^i,a_h^i)$  is a $H$-subgaussian random variable. By the Assumption \ref{assumption-linear-mixture}, for the vector $\bphi_{{\vvalue}_{h+1}^i}(s_h^i,a_h^i)$, we have $\|\bphi_{{\vvalue}_{h+1}^i}(s_h^i,a_h^i)\|_2\leq H$ and for the parameter $\btheta_h^*$, we have $\|\btheta_h^*\|_2\leq C_{\theta}$. Therefore, by Theorem 2 in \citet{abbasi2011improved}, with probability at least $1-\delta/H$, for all episode $k\in[K]$, we have:
\begin{align}
    \btheta_h^*\in \bigg\{\btheta\in \RR^d:\|\btheta_h^* - \btheta_{k,h}\|_{\bSigma_h^k} \leq H\sqrt{d\log\frac{H+kH^3/\lambda}{\delta}} + \sqrt{\lambda}C_{\theta}\bigg\}.\notag
\end{align}
Finally, by the definition of $\beta_k$ and taking an union bound for all step $h\in[H]$, we finish the proof of Lemma \ref{lemma: confidence-set}. 
\end{proof}

\subsection{Proof of Lemma \ref{lemma:jia-transition}}
\begin{proof}[Proof of Lemma \ref{lemma:jia-transition}]
For each step $h\in[H]$ and episode $k\in[K]$, by the definition of the value function $\vvalue_h^k(s_h^k)$, we have
\begin{align}
    \qvalue_h^k(s,a)&=\reward(s,a)+\bphi_{\vvalue_{h+1}^{k}}(s,a)^{\top}\btheta_{k,h}+\beta_{k}\sqrt{\big(\bphi_{\vvalue_{h+1}^{k}}(s,a)\big)^{\top}(\bSigma_{h}^{k})^{-1}\bphi_{\vvalue_{h+1}^k}(s,a)}\notag\\
    &=\reward(s,a)+\bphi_{\vvalue_{h+1}^{k}}(s,a)^{\top}\btheta^*_{h}+\beta_{k}\sqrt{\big(\bphi_{\vvalue_{h+1}^{k}}(s,a)\big)^{\top}(\bSigma_{h}^{k})^{-1}\bphi_{\vvalue_{h+1}^k}(s,a)}+\bphi_{\vvalue_{h+1}^{k}}(s,a)^{\top}(\btheta_{k,h}-\btheta^*_{h})\notag\\
    &=\reward(s,a)+[\PP_h\vvalue_{h+1}^{k}](s,a)+\beta_k \|\bphi_{\vvalue_{h+1}^{k}}\|_{(\bSigma_{h}^{k})^{-1}}+\bphi_{\vvalue_{h+1}^{k}}(s,a)^{\top}(\btheta_{k,h}-\btheta^*_{h}).\label{eq:71}
\end{align}
When the result of Lemma \ref{lemma: confidence-set} holds, we have
\begin{align}
    \big|\bphi_{\vvalue_{h+1}^{k}}(s,a)^{\top}(\btheta_{k,h}-\btheta^*_{h})\big|\leq \|\bphi_{\vvalue_{h+1}^{k}}\|_{(\bSigma_{h}^{k})^{-1}}+\|\btheta_{k,h}-\btheta^*_{h}\|_{(\bSigma_{h}^{k})}\leq \beta_k \|\bphi_{\vvalue_{h+1}^{k}}\|_{(\bSigma_{h}^{k})^{-1}}\label{eq:72},
\end{align}
where the first inequality holds due to Cauchy-Schwarz inequality and the second inequality holds due to Lemma \ref{lemma: confidence-set}. Substituting \eqref{eq:72} into \eqref{eq:71}, we have
\begin{align}
    \reward(s,a)+[\PP_h\vvalue_{h+1}^{k}](s,a)\leq\qvalue_h^k(s,a)\leq \reward(s,a)+[\PP_h\vvalue_{h+1}^{k}](s,a)+2\beta_k \|\bphi_{\vvalue_{h+1}^{k}}\|_{(\bSigma_{h}^{k})^{-1}}.\label{eq:73}
\end{align}
Furthermore, by the Bellman equation, we have $\qvalue_h^{\pi_k}(s,a) = \reward_h(s,a) + [\PP_h\vvalue_{h+1}^{\pi_k}](s,a)$ and \eqref{eq:73} implies
\begin{align}
    \qvalue_h^k(s,a)-\qvalue_h^{\pi_k}(s,a)\leq\big[\PP(\vvalue_{h+1}^{k} -\vvalue_{h+1}^{\pi_k})\big](s,a)+2\beta_k \sqrt{\bphi(s,a)^{\top}(\Lambda_1^k)^{-1}\bphi(s,a)}.\notag
\end{align}
By the Bellman equation for the optimal policy, we have $\qvalue_h^{*}(s,a) = \reward_h(s,a) + [\PP_h\vvalue_{h+1}^{*}](s,a)$ and \eqref{eq:73} also we implies that
\begin{align}
    \qvalue_h^k(s,a)-\qvalue_h^{*}(s,a)\ge  \big[\PP_h(\vvalue_{h+1}^{k} -\vvalue_{h+1}^{*})\big](s,a).\label{eq:74}
\end{align}
\eqref{eq:74} shows that when $\vvalue_{h+1}^{k}(s) -\vvalue_{h+1}^{*}(s)\ge 0$ holds for all state $s\in\cS$ at step $h+1$, then for all $s\in \cS,a\in \cA$, we also have $\qvalue_h^k(s,a)\ge \qvalue_h^*(s,a)$ and $\vvalue_h^k(s)\ge \vvalue_h^*(s)$ at step $h$. 
Since $\vvalue_{H+1}^{*}(s,a)=\vvalue_{H+1}^{k}(s,a)=0$, it is easily to show that for each step $h\in[H]$ and all $s\in \cS,a\in \cA$, we have $\qvalue_h^k(s,a)\ge \qvalue_h^*(s,a)$ and $\vvalue_h^k(s)\ge \vvalue_h^*(s)$. Thus, we finish the proof of Lemma \ref{lemma:jia-transition}.
\end{proof}

\subsection{Proof of Lemma \ref{lemma:jia-sum}}
\begin{proof}[Proof of Lemma \ref{lemma:jia-sum}]
For simplicity, we denote 
\begin{align}
\bSigma'_i=\lambda \Ib+\sum_{j=1}^i(\bphi'_i)^{\top}\bphi'_i,\notag
\end{align}
where $\bphi'_i$ is the abbreviation of $\bphi_{\vvalue_{h+1}^{c_i}}(s_h^{c_i},a_h^{c_i})$. Thus, we have
\begin{align}
    \sum_{i=1}^{k} (\bphi'_i)^{\top}(\bSigma_h^{c_i})^{-1}\bphi'_i&\leq \sum_{i=1}^{k} (\bphi'_i)^{\top}(\bSigma'_{i-1})^{-1}\bphi'_i \leq 2\log \bigg[\frac{\det(\bSigma'_k)}{\det(\bSigma'_0)}\bigg] \leq 2d\log(1+k),\notag
\end{align}
where the first inequality holds due to $\bSigma'_{i-1}\preceq \Lambda_1^{c_i}$, the second inequality holds due to Lemma \ref{Lemma:abba} and the last inequality holds due to $\|\bSigma'_k\|=\|H^2d \Ib+\sum_{i=1}^k(\bphi'_k)^{\top}\bphi'_k\|\leq H^2d+H^2k.$
\end{proof} 

\subsection{Proof of Lemma \ref{lemma:jia-gap-number}}
\begin{proof}[Proof of Lemma \ref{lemma:jia-gap-number}]

\noindent We fix $h$ in this proof. Let $k_0=0$, and for $i\in[N]$, we denote $k_i$ as the minimum index of the episode where the sub-optimality at step $h$ is no less than $2^n\text{gap}_{\text{min}}$:
\begin{align}
    k_i&=\min \big\{k: k>k_{i-1}, \vvalue_h^*(s_h^k)-\qvalue_h^{\pi_k}(s_h^k,a_h^k)\ge 2^n\text{gap}_{\min}\big\}.\label{eq:jia-ti}
\end{align}
For simplicity, we denote by $K'$ the number of episodes such that the sub-optimality of this episode at step $h$ is no less than $2^n\text{gap}_{\text{min}}$. Formally speaking, we have
\begin{align}
    K'=\sum_{k=1}^K  \ind\big[\vvalue_h^*(s_h^k)-\qvalue_h^{\pi_k}(s_h^k,a_h^k)\ge 2^n\text{gap}_{\min}\big].\notag
\end{align}
By the definition of $K'$, we have
\begin{align}
    \sum_{i=1}^{K'}\big(\qvalue_h^{k_i}(s_h^{k_i},a_h^{k_i})-\qvalue_h^{\pi_{k_i}}(s_h^{k_i},a_h^{k_i})\big)& \ge \sum_{i=1}^{K'}\Big(\qvalue_h^{k_i}\big(s_h^{k_i},\pi_h^*(s_h^{k_i},h)\big)-\qvalue_h^{\pi_{k_i}}(s_h^{k_i},a_h^{k_i})\Big)\notag\\
    & \ge \sum_{i=1}^{K'}\Big(\qvalue_h^*\big(s_h^{k_i},\pi_h^*(s_h^{k_i},h)\big)-\qvalue_h^{\pi_{k_i}}(s_h^{k_i},a_h^{k_i})\Big)\notag\\
    & = \sum_{i=1}^{K'}\big(\vvalue_h^*(s_h^{k_i})-\qvalue_h^{\pi_{k_i}}(s_h^{k_i},a_h^{k_i})\big)\notag\\
    & \ge 2^n\text{gap}_{\min} K',\label{eq:jia-lower-Regret}
\end{align}
where the first inequality holds due to the definition of policy $\pi^{k_i}$, the second inequality holds due to Lemma \ref{lemma:jia-transition} and the last inequality hold due to the definition of $k_i$ in \eqref{eq:jia-ti}.

In other hand, for $h'\in [H], k\in [K]$, we have
\begin{align}
    \qvalue_{h'}^k(s_{h'}^k,a_{h'}^k)-\qvalue_{h'}^{\pi_k}(s_{h'}^k,a_{h'}^k)
    & \leq \big[\PP_h(\vvalue_{h'+1}^{k} -\vvalue_{h'+1}^{\pi_{k}})\big](s_{h'}^k,a_{h'}^k) +2\beta_k \sqrt{\bphi(s_{h'}^k,a_{h'}^k)^{\top}(\bSigma_{h'}^k)^{-1}\bphi(s_{h'}^k,a_{h'}^k)}\notag\\
    & = \vvalue_{h'+1}^{k}(s_{h'+1}^k) -\vvalue_{h'+1}^{\pi_{k}}(s_{h'+1}^k)+\epsilon_{h'}^k+2\beta_k \sqrt{\bphi(s_{h'}^k,a_{h'}^k)^{\top}(\bSigma_{h'}^k)^{-1}\bphi(s_{h'}^k,a_{h'}^k)}\notag\\
    &= \qvalue_{h'+1}^k(s_{h'+1}^k,a_{h'+1}^k)-\qvalue_{h'+1}^{\pi_k}(s_{h'+1}^k,a_{h'+1}^k)+\epsilon_{h'}^k\notag\\
    & \quad+2\beta_k \sqrt{\bphi(s_{h'}^k,a_{h'}^k)^{\top}(\bSigma_{h'}^k)^{-1}\bphi(s_{h'}^k,a_{h'}^k)}\label{eq:jia-tele},
\end{align}
where 
\begin{align}
    \epsilon_{h'}^{k}&= \big[\PP_h(\vvalue_{h'+1}^{k} -\vvalue_{h'+1}^{\pi_{k}})\big](s_{h'}^k,a_{h'}^k)-\big(\vvalue_{h'+1}^{k}(s_{h'+1}^k) -\vvalue_{h'+1}^{\pi_{k}}(s_{h'+1}^k)\big),\notag
\end{align} and the inequality holds due to Lemma \ref{lemma:jia-transition}.
Take summation for \eqref{eq:jia-tele} over all $k_i$ and $h\leq h'\leq H$, we have
\begin{align}
    &\sum_{i=1}^{K'}\big(\qvalue_h^{k_i}(s_h^{k_i},a_h^{k_i})-\qvalue_h^{\pi_{k_i}}(s_h^{k_i},a_h^{k_i})\big) \leq \underbrace{\sum_{i=1}^{K'}\sum_{h'=h}^{H}2\beta_k \sqrt{\bphi(s_{h'}^{k_i},a_{h'}^{k_i})^{\top}(\bSigma_{h'}^{k_i})^{-1}\bphi(s_{h'}^{k_i},a_{h'}^{k_i})}}_{I_1} + \underbrace{\sum_{i=1}^{K'}\sum_{h'=h}^{H}\epsilon_{h'}^{k_i}}_{I_2}.\label{eq:jia-upper-Regret}
\end{align}
For term $I_1$, we have
\begin{align}
    I_1&=\sum_{i=1}^{K'}\sum_{h'=h}^{H}2\beta_k \sqrt{\bphi(s_{h'}^{k_i},a_{h'}^{k_i})^{\top}(\bSigma_{h'}^{k_i})^{-1}\bphi(s_{h'}^{k_i},a_{h'}^{k_i})}\notag\\
    &\leq 2\beta_K\sqrt{K'}\sum_{h'=h}^{H} \sqrt{\sum_{i=1}^{K'}\bphi(s_{h'}^{k_i},a_{h'}^{k_i})^{\top}(\bSigma_{h'}^{k_i})^{-1}\bphi(s_{h'}^{k_i},a_{h'}^{k_i}) }\notag\\
    &\leq 2H\beta_K \sqrt{K'}\sqrt{2d\log(K'+1)},\label{eq:jia-I_1}
\end{align}
where the first inequality holds due to Cauchy-Schwarz inequality and the  second inequality holds due to Lemma \ref{lemma:jia-sum}.

For each $k\in [K]$, by Lemma \ref{lemma:azuma},  with probability at least $1-\delta$, we have
\begin{align}
    &\sum_{i=1}^k \sum_{j=h}^{H}\Big(\big[\PP_j(\vvalue^{k_i}_{j+1}-\vvalue^{\pi_{k_i}}_{j+1})\big](s^{k_i}_{j},a^{k_i}_{j})-\big(\vvalue^{k_i}_{j+1}(s^{k_i}_{j+1})-\vvalue^{\pi_{k_i}}_{j+1}(s^{k_i}_{j+1})\big)\Big)\notag \leq \sqrt{2kH^2\log(1/\delta)}.\notag
\end{align}
Thus, taking a union bound for all $k\in[K]$, with probability at least $1-K\delta$, we can bound term $I_2$ as follows:
\begin{align}
    &\sum_{i=1}^{K'} \sum_{j=h}^{H}\Big(\big[\PP_j(\vvalue^{k_i}_{j+1}-\vvalue^{\pi_{k_i}}_{j+1})\big](s^{k_i}_{j},a^{k_i}_{j})-\big(\vvalue^{k_i}_{j+1}(s^{k_i}_{j+1})-\vvalue^{\pi_{k_i}}_{j+1}(s^{k_i}_{j+1})\big)\Big) \leq \sqrt{2K'H^2\log(1/\delta)}.\label{eq:jia-I_3}
\end{align}
Substituting \eqref{eq:jia-I_1} and \eqref{eq:jia-I_3} into \eqref{eq:jia-upper-Regret}, with probability at least $1-(K+1)\delta,$ we have
\begin{align}
     &\sum_{i=1}^{K'}\big(\qvalue_h^{k_i}(s_h^{k_i},a_h^{k_i})-\qvalue_h^{\pi_{k_i}}(s_h^{k_i},a_h^{k_i})\big)\leq  \sqrt{2K'H^2\log(1/\delta)}+ 2H\beta_K \sqrt{K'}\sqrt{2d\log(K'+1)}.\label{eq:jia-Upper}
\end{align}
Combining \eqref{eq:jia-Upper} and \eqref{eq:jia-lower-Regret}, we have
\begin{align}
    2^n\text{gap}_{\min}K'\leq \sqrt{2K'H^2\log(1/\delta)}+ 2H\beta_K \sqrt{2K'd\log(K'+1)},\notag
\end{align}
which implies
\begin{align}
  K'\leq \frac{512C^2_{\btheta}d^2H^4 \log^3(2dHK/\delta)}{4^n\text{gap}^2_{\min}}\log \bigg(\frac{512C^2_{\btheta}d^2H^4 \log^3(2dHK/\delta)}{4^n\text{gap}^2_{\min}}\bigg). \notag
\end{align}
\end{proof}

\subsection{Proof of Lemma \ref{lemma:jia-gap-sum}}
\begin{proof}[Proof of Lemma \ref{lemma:jia-gap-sum}]
By the definition of $\text{gap}_{\min}$ in \eqref{definition-gap-min}, for each $h\in[H], k\in[K]$, we have $\vvalue_h^*(s_h^k)-\qvalue_h^{*}(s_h^k,a_h^k)=0$ or  $\vvalue_h^*(s_h^k)-\qvalue_h^{*}(s_h^k,a_h^k)\ge \text{gap}_{\min}$. Thus, we divide the interval       $[\text{gap}_{\min},H]$ to $N=\big\lceil \log(H/ \text{gap}_{\min})\big\rceil$ intervals: $\big[2^{i-1}\text{gap}_{\min},2^i\text{gap}_{\min}\big) \big(i\in[N]\big)$ and with probability at least $ 1-2(K+1)\log(H/\text{gap}_{\min})\delta$, we have
\begin{align}
    \sum_{k=1}^K  \big(\vvalue_h^*(s_h^k)-\qvalue_h^{*}(s_h^k,a_h^k)\big)\notag &\leq \sum_{i=1}^N\sum_{k=1}^K \ind\big[2^{i}\text{gap}_{\min}\ge \vvalue_h^*(s_h^k)-\qvalue_h^*(s_h^k,a_h^k)\ge 2^{i-1}\text{gap}_{\min}\big]\times 2^i\text{gap}_{\min}\notag\\
    & \leq \sum_{i=1}^N\sum_{k=1}^K \ind\big[ \vvalue_h^*(s_h^k)-\qvalue_h^{\pi_k}(s_h^k,a_h^k)\ge 2^{i-1}\text{gap}_{\min}\big]\times 2^i\text{gap}_{\min}\notag\\
    & \leq \sum_{i=1}^N   \frac{2048C^2_{\btheta}d^2H^4 \log^3(2dHK/\delta)}{2^{i}\text{gap}_{\min}}\log \bigg(\frac{512C^2_{\btheta}d^2H^4 \log^3(2dHK/\delta)}{4^{i-1}\text{gap}^2_{\min}}\bigg)\notag\\
    & \leq \frac{2048C^2_{\btheta}d^2H^4 \log^3(2dHK/\delta)}{\text{gap}_{\min}}\log \bigg(\frac{512C^2_{\btheta}d^2H^4 \log^3(2dHK/\delta)}{\text{gap}^2_{\min}}\bigg)\notag,
\end{align}
where the first inequality holds due to $\vvalue_h^*(s_h^k)-\qvalue_h^{*}(s_h^k,a_h^k)\ge \text{gap}_{\min}$ or $\vvalue_h^*(s_h^k)-\qvalue_h^{*}(s_h^k,a_h^k)=0$, the second inequality holds on due to $\vvalue_h^*(s_h^k)-\qvalue_h^{\pi_k}(s_h^k,a_h^k)\ge \vvalue_h^*(s_h^k)-\qvalue_h^*(s_h^k,a_h^k)$, the third inequality holds due to due to Lemma \ref{lemma:jia-gap-number}. Thus, we finish the proof of Lemma \ref{lemma:jia-gap-sum}.
\end{proof}

\section{Auxiliary Lemmas}

\begin{lemma}[Azuma–Hoeffding inequality \citealt{cesa2006prediction}]\label{lemma:azuma}
Let $\{x_i\}_{i=1}^n$ be a martingale difference sequence with respect to a filtration $\{\cG_{i}\}$ satisfying $|x_i| \leq M$ for some constant $M$, $x_i$ is $\cG_{i+1}$-measurable, $\EE[x_i|\cG_i] = 0$. Then for any $0<\delta<1$, with probability at least $1-\delta$, we have 
\begin{align}
    \sum_{i=1}^n x_i\leq M\sqrt{2n \log (1/\delta)}.\notag
\end{align} 
\end{lemma}
\begin{lemma}[Lemma 11 in \citealt{abbasi2011improved}]\label{Lemma:abba}
Let $\{X_t\}_{t=1}^{+\infty}$ be a sequence in $\RR^d$, $V$ a $d \times d$ positive definite matrix and define $V_t=V+\sum_{i=1}^{t} X_t^{\top}X_t$. If $\|X_t\|_2\leq L$ and $\lambda_{\min}(V)\ge \max(1,L^2),$ then we  have
\begin{align}
    \sum_{i=1}^{t} X_i^{\top} (V_{i-1})^{-1} X_i\leq 2 \log \bigg(\frac{\det{V_t}}{\det{V}}\bigg).\notag
\end{align}
\end{lemma}
\begin{lemma}[Freedman inequality, \citealt{cesa2006prediction}]\label{lemma:freedman}
Let $\{x_i\}_{i=1}^n$ be a martingale difference sequence with respect to a filtration $\{\cG_{i}\}$ satisfying $|x_i| \leq M$ for some constant $M$, $x_i$ is $\cG_{i+1}$-measurable, $\EE[x_i|\cG_i] = 0$ and define that $V=\sum_{i=1}^n \EE(x_i^2|\cG_{i}).$
 Then for any $a>0,v>0$, we have
\begin{align}
\Pr\Big(\sum_{i=1}^n x_i \ge a,V\leq v\Big)\leq  \exp\Big(\frac{-a^2}{2v+2aM/3}\Big).\notag 
\end{align}
\end{lemma}

\bibliographystyle{ims}
\bibliography{reference}

\end{document}